\documentclass[letterpaper,11pt]{article}
\usepackage[utf8]{inputenc}
\usepackage{times}  
\usepackage{helvet} 
\usepackage{courier}  
\usepackage[hyphens]{url}  
\usepackage{graphicx} 
\usepackage[ bottom=1.25in, top=1.25in, left=1.25in, right=1.25in]{geometry}
\usepackage{natbib} 
\usepackage{caption} 
\usepackage{authblk}
\usepackage[dvipsnames]{xcolor}
\usepackage[colorlinks = true,
            linkcolor = blue,
            urlcolor  = blue,
            citecolor = LimeGreen,
            anchorcolor = blue]{hyperref}

\title{Fair Active Ranking from Pairwise Preferences}
\date{}
\author[$1$]{Sruthi Gorantla}
\author[$2$]{Sara Ahmadian}
\affil[$1$]{\small{Indian Institute of Science, Bengaluru, India. \texttt{gorantlas@iisc.ac.in}}}
\affil[$2$]{\small{Google Research. \texttt{sahmadian@google.com}}}

\usepackage{booktabs} 
\usepackage{amssymb}
\usepackage{xcolor}
\usepackage{amsmath}
\usepackage{amsthm}
\usepackage{thmtools}
\usepackage[switch]{lineno}
\usepackage{caption}
\usepackage{enumitem}
\usepackage{bbm}
\usepackage{multicol}

\usepackage[capitalize,noabbrev]{cleveref}
\usepackage{algorithm}
\usepackage{algorithmic}

\usepackage{thmtools}
\usepackage{thm-restate}

\newtheorem{theorem}{Theorem}
\newtheorem{lemma}[theorem]{Lemma}

\newtheorem{definition}[theorem]{Definition}

\newtheorem*{remark}{Remark}

\newcommand{\R}{{\mathbb R}}

\newcommand{\E}{{\mathbf E}}

\newcommand{\cA}{{\mathcal A}}

\newcommand{\cB}{{\mathcal B}}

\newcommand{\cI}{{\mathcal I}}

\newcommand{\cE}{{\mathcal E}}

\newcommand{\cF}{{\mathcal F}}

\newcommand{\cO}{{\mathcal O}}

\newcommand{\X}{{\mathcal X}}

\newcommand{\tS}{{\tilde {S}}}
\newcommand{\sS}{{S^*}}

\newcommand{\p}{{\mathbf p}}

\newcommand{\teps}{{\tilde \epsilon}}
\newcommand{\sm}{{\setminus}}
\newcommand{\blind}{\texttt{group-blind}}
\newcommand{\aware}{\texttt{group-aware}}
\newcommand{\pop}{\textsc{Pop}}

\newcommand{\efair}{\err^{\texttt{fair}}}

\def \algbtp{\textsc{Beat-the-Pivot}} 
\def \pac{{\bf $(\epsilon,\delta)$-PAC-Ranker}}
\def \pacf{{\bf $(\epsilon,\delta)$-PACF-Ranker}}
\def \ebr{{\bf $\epsilon$-Best-Ranking}}

\def \epqbr{{\bf $\epsilon$-Best-}{\sc Fair}{\bf-Ranking}}

\newcommand{\bnu}{{\boldsymbol \nu}}

\newcommand{\bsigma}{\boldsymbol \sigma}
\newcommand{\paren}[1]{\left(#1\right)}
\newcommand{\sparen}[1]{\left[#1\right]}

\newcommand{\err}{\text{err}}

\makeatother
\allowdisplaybreaks

\begin{document}
\maketitle
\begin{abstract}
  We investigate the problem of probably approximately correct and \textit{fair} (PACF) ranking of items by adaptively evoking pairwise comparisons.  Given a set of $n$ items that belong to disjoint groups, our goal is to find an $(\epsilon, \delta)$-PACF-Ranking according to a fair objective function that we propose. We assume access to an oracle, wherein, for each query, the learner can choose a pair of items and receive stochastic winner feedback from the oracle.
  Our proposed objective function asks to minimize the $\ell_q$ norm of the error of the groups, where the error of a group is the $\ell_p$ norm of the error of all the items within that group, for $p, q \geq 1$.  
  This generalizes the objective function of $\epsilon$-Best-Ranking, proposed by \citet{saha19a}.
  
  By adopting our objective function, we gain the flexibility to explore fundamental fairness concepts like equal or proportionate errors within a unified framework. Adjusting parameters $p$ and $q$ allows tailoring to specific fairness preferences.
  We present both group-blind and group-aware algorithms and analyze their sample complexity. We provide matching lower bounds up to certain logarithmic factors for group-blind algorithms.
  For a restricted class of group-aware algorithms, we show that we can get reasonable lower bounds.
  We conduct comprehensive experiments on both real-world and synthetic datasets to complement our theoretical findings.
\end{abstract}

\section{Introduction}

Ranking is a fundamental problem in data mining and machine learning that arises in a wide range of applications, such as search engines, recommender systems, and information retrieval.  The simplest and most extensively studied version of ranking uses noisy pairwise comparisons, as initiated by \citet{feige1994computing}. Recently, it has also been studied with dueling bandits \citep{Busa2014}. 

Ranking problems have been studied under fairness constraints to mitigate or eliminate the bias and discrimination in the solutions constructed by the existing algorithms. The primary focus of these works is on generating rankings with respect to socially salient attributes \citep{zehlike2022fairness1,zehlike2022fairness2,pitoura2022fairness} that asks for equal or proportional representation within every prefix of the ranking. Although such fairness constraints are especially useful in human-centric applications, such as hiring, credit allocation, recidivism prediction, and college admissions, they may still produce adverse outcomes for particular groups of individuals. This is primarily because the error in ranking is measured as an aggregate of errors over all the items. The resulting ranking may systematically discriminate against minority groups in terms of error, even after applying representation fairness constraints. The error function does not take into account biases in the data or differences in the data distributions for different demographic groups. 

We introduce a fair ranking inspired by social fairness concepts, which focuses on generating rankings that fairly distribute error across groups. Our metric is versatile in the definition of the error within each group and in the aggregation of the errors over different groups. This yields a metric that generalizes several well-known notions of fairness.

In this paper, we study active, PAC ranking of $n$ items using pairwise comparisons. In this setting, the learner receives preference feedback, for requested pairwise comparisons, according to the well-known Plackett-Luce (PL) probability model. The learner’s goal is to find a near-optimal ranking, with respect to tolerance parameter $\epsilon$, with high probability $(1-\delta)$, using as few pairwise comparison rounds as possible. Overall, our contributions can be summarized as follows:
\begin{enumerate}[leftmargin=*,itemsep=0mm]
    \item We introduce a fair variant of the error metric for rankings based on social fairness concepts (Definition~\ref{def:ours}). To this end, we also discuss the limitations of previous fair ranking definitions (Section~\ref{sec:prelims}), and how algorithms optimized for our proposed objective function overcome these limitations.
    \item We study two classes of algorithms depending on whether they have access to the group labels of the items or not. We call the algorithms without access \blind~and those with access \aware. We design efficient algorithms output probably approximately correct rankings for our fair objective function and analyse their sample complexity bounds. We provide matching lower bounds up to certain logarithmic factors for group-blind algorithms.
    For a restricted class of group-aware algorithms, we show that we can get reasonable lower bounds.
    \item 
    We empirically evaluate our algorithms on real-world and synthetic datasets and show that our \aware~algorithm has significantly lower sample complexity than the \blind. We also show that our \aware~algorithm achieves lower overall error as well as lower error on all the groups than \blind~algorithm. 
\end{enumerate}
In \Cref{sec:related}, we discuss several related works. In \Cref{sec:prelims}, we introduce our metric and define PAC ranking for our metric. In Sections~\ref{sec:theoretical_results} and \ref{sec:experimental_results}, we show our main theoretical and experimental results respectively.

\section{Related Works}
\label{sec:related}
\paragraph{Fairness in ranking.} Group fairness notions ask for groups to be treated equally, such as asking for equality of opportunity in supervised learning \citep{hardt2016}, equitable clustering costs across groups in clustering \citep{ghadiri21,ABV2021,chlamtac22,GKDL23}, and equal representation of the groups in ranking \citep{CSV2018} or subset selection \citep{kleinberg2018}.
On the contrary, individual fairness treats fairness at the individual level and not as some aggregate function of groups. It asks for similar individuals to be treated similarly for the task at hand \citep{DHPRZ2012}. 
Particularly in ranking, the group fairness notions include ensuring sufficient representation of all the groups in each prefix of the top $k$ ranking \citep{CSV2018,ZBCHMB2017,zehlike2022fairness1} or every $k$ consecutive ranks \citep{gorantla21a}.
Other works ask for equality of exposure of the groups \citep{SJ2018}, fair ranking under uncertain merit scores \citep{SKJ2021}, fair ranking with noisy sensitive attributes \citep{mehrotra22}, fair ranking in the presence of implicit bias \citep{CMV2020}, etc.
However, these are all algorithmic solutions that assume access to the merit (or relevance) scores of the items and maximize some objective function such as {\textsf{NDCG, Precision@$k$}}.
In contrast, we assume that we only have oracle access to pairwise preferences of the items and study the problem of inferring a ranking with a minimal number of queries with high confidence $1-\delta$, and up to an error $\epsilon$, while ensuring fairness in errors incurred by the groups. Even though previous works such as \citep{saha19a} studied non-fair variant of the problem in the pairwise preference model, to the best of our knowledge, we are the first to study this with fairness consideration.
\paragraph{Cascaded norms for fairness.} 
Cascaded norm objectives have been used in generalizing the cost-based objective functions to account for different costs borne by different groups. \textit{Metric space clustering} objective asks for minimizing the $\ell_p$-norm of the distances between points in a cluster and their center. Most interesting cases are when $p \in \{1, 2, \infty\}$ as they correspond  to the well studied $k$-median, $k$-means, and $k$-center, respectively.
\cite{chlamtac22} later generalized this further to account for group fairness. They ask to minimize the $\ell_q$-norm of the cost of the groups, where cost of a group is the $\ell_p$-norm of the distances between the points in a cluster from that group and their cluster center.
This generalizes the fair clustering notions such as \textit{Socially Fair} $k$-means and $k$-medians clustering.
This objective function allows us to treat many problems under one umbrella.
Taking inspiration from this, we propose the notion of \epqbr, using norms of errors in a cascaded fashion, first within the groups and then across the groups.

\paragraph{Active ranking from pairwise comparisons.}
Pairwise comparisons are well-motivated for the sake of ranking as they provide an easier way of collecting peoples' preferences; it is easier to compare two candidates for a job rather than assign an absolute score to them independently and compare the scores to get relative preferences. Even mathematically speaking, pairwise preferences enforce a weaker constraint on the data collection process than asking for exact scores because the former need not satisfy transitivity in the pairwise preferences of the items while the latter induces transitivity on the items. 
The \textit{passive} sample complexity for ranking with pairwise comparisons has been resolved under different model assumptions in recent works \citep{Gleich11,rajkumar16}.
The exact sample complexity depends on the objective function defined to measure the ``goodness'' of the ranking.
However, collecting pairwise preference labels might be expensive.
Active learning has become a prevailing technique for designing efficient supervised learning algorithms, where one only needs to query labels that are most “informative”.
In many settings, active learning gives an exponential improvement in the sample complexity compared to its \textit{passive} counterparts under some weak distributional assumptions \citep{BL13}.
Recent works have studied ranking in the \textit{active} or adaptive setting \citep{Ailon2012AnAL,jamieson11,saha19a,RLS21}, among which \cite{RLS21, saha19a} studied PAC sample complexity bounds in a multi-wise comparisons feedback model. 
There have also been works on estimating the parameters of the pairwise preference model that generates the rankings \cite{khetan16}.


\section{Preliminaries}
\label{sec:prelims}
\paragraph{Notation.}
For any positive integer $t$, we use $[t]$ to denote the set of integers $\{1,2,\ldots, t\}$.
Let $\cI$ represent the set of items and the set $\Sigma_\cI$ to contain all possible permutations of the items in $\cI$.
Then for any permutation $\bsigma \in \Sigma_\cI$, $\bsigma(j)$ represents the index of item in $j$-th position of the permutation.
We use $\bsigma^{-1}(i)$ to represent the position item $i$ is assigned to in the permutation.
We assume that the set of items can be partitioned into $\gamma$ disjoint groups based on socially salient features such as age, race, gender, etc, denoted as $G_1, G_2, \ldots, G_\gamma$. Let $n_h := |G_h|, \forall h \in [\gamma]$ and $n := |\cI|$.
Therefore, $n = \sum_{h \in [\gamma]} n_h$.

\paragraph{The active ranking problem.}
Given a set $\cI$ of items, a ranking $\bsigma$ defines a total ordering of the items. Let us assume that each item $i \in \cI$ is associated with a true relevance score $\theta_i \in \R$, where the higher the value of $\theta_i$, the better to rank it at the top.
There is only one ranking that is consistent with the true scores -- the one that ranks the items in the descending order of their true scores.
However, we consider the setup where the true scores are not directly available, but any pair of items can be compared via an oracle, which has access to their true scores.
Depending on the type of feedback from the oracle, the goal is to find a ranking of the items by \textit{actively} querying for pairwise comparisons of the items one after the other.
It is useful to note that sorting $n$ items based on pairwise comparisons fits in this problem setup, which needs $\Theta(n\log n)$ actively chosen pairwise comparison queries.
Note that the pairwise comparison queries are fixed up front in the \textit{passive} setup; hence, the output of one query does not affect what other queries are made by the algorithm, unlike the active setup.

\paragraph{Plackett-Luce (PL) feedback oracle.}
PL model is extensively used in generating stochastic rankings \cite{rajkumar16,saha19a,SKJ2021}.
The probability of sampling a ranking in the PL model is given by,
\[
\Pr[\bsigma | \theta_1, \ldots, \theta_n] = \prod_{i \in [n]}\frac{\theta_{\sigma^{-1}(i)}}{\sum_{j = i}^n\theta_{\sigma^{-1}(j)}}.
\]
An appealing property of the PL model is that the pairwise winner probabilities for a pair of items $\{i,i'\} \in \cI$ are very easy to calculate. That is,
\[
\Pr[i\mid \{i,i'\}] = \frac{\theta_i}{\theta_i+\theta_{i'}}~~\text{and}~~\Pr[i'\mid \{i,i'\}] = \frac{\theta_{i'}}{\theta_i+\theta_{i'}}.
\]
\cite{saha19a} also studied the PAC sample complexity bounds under the PL model to leverage the property of independence of irrelevant attributes satisfied by the PL model, which helps in consistently aggregating pairwise preferences to find a total ordering. Hence, we also study our problem under the PL model.
\paragraph{Ranking performance.}
In a ranking $\bsigma$, each item $i$ suffers an error based on the items incorrectly ordered above $i$.
One way of quantifying this error is as follows,
\[
d_i(\bsigma; \theta)  := \max_{\substack{i' \in [n]~\text{s.t.}\\ \theta_i > \theta_{i'} \land \bsigma^{-1}(i) > \bsigma^{-1}(i')}} \theta_i - \theta_{i'}.
\]
If there are no items with higher scores ranked after $i$, we define the error for $i$ to be zero.
Then, the performance of a candidate ranking $\bsigma \in \Sigma_{\cI}$ can be measured as an aggregate of the errors of the items.
One such metric proposed by \cite{saha19a} is the $\ell_{\infty}$ norm of the errors of the items. 
Using this, \cite{saha19a} define a ranking $\bsigma \in \Sigma_{\cI}$ to be an \ebr~iff,
\begin{equation}
\err(\bsigma;\theta)  := \max_{i \in \cI}~~d_i(\bsigma;\theta) < \epsilon. \label{eq:theirs}
\end{equation}
A $0$-Best-Ranking is called a Best-Ranking or optimal ranking of the PL model.

\begin{figure}
    \centering
    \includegraphics[width=0.5\textwidth]{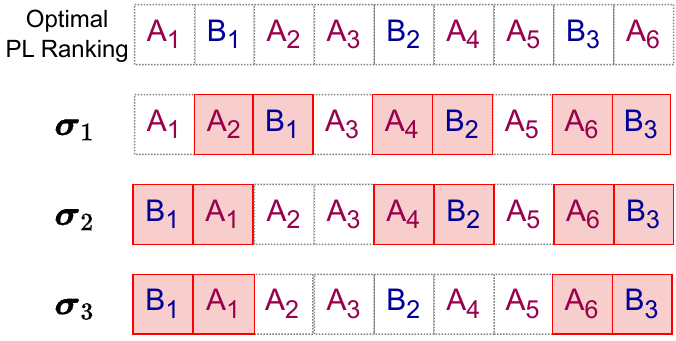}
    \caption{Example of different $\epsilon$-Best-Rankings.}
    \label{fig:example}
\end{figure}

We will now illustrate the shortcomings of this metric with an example (see \Cref{fig:example}).
Let there be a set of $9$ items such that $6$ items belong to \textit{group-A} (let us call this the \textit{majority} group) and $3$ items belong to \textit{group-B} (\textit{minority} group).
Let their true scores be $\theta_i = 1-((i-1)*0.09)$ for each $i = 1,2,\ldots,10$.
Then, their optimal PL ranking is as shown in Figure~\ref{fig:example}.
We use $A_t$ and $B_t$ to represent $t$-th item from group-$A$ and group-$B$, respectively, in the order in which they appear in the optimal PL ranking.
For $\epsilon = 0.09$ all of $\bsigma_1, \bsigma_2,$ and $\bsigma_3$ are $\epsilon$-Best-Rankings.
However, in $\bsigma_1$, all the error is borne by group-$B$, and in $\bsigma_2$, even though both the groups incur errors, the errors are unequal.
Hence, \ebr~does not guarantee a fair distribution of error across different groups of items.
$\bsigma_3$ is a good solution as both the groups incur an error of $\epsilon$. 
Hence, we need to optimize the algorithms for an objective function that distributes error fairly across groups.
We also note here that $\bsigma_1$ also satisfies the notion of proportional representation of the groups proposed in \citep{gorantla21a} since in every $3$ consecutive ranks, there are $2$ items from group-$A$ and $1$ item from group-$B$. Hence, achieving a proportional representation of groups in ranking may not be sufficient to ensure a fair distribution of error across groups.

\paragraph{Proposed metric.}
We propose to use a parameterized variant of the error, namely, using the $\ell_q$ norm of the group-wise error, where the group-wise error is nothing but the $\ell_p$-norm of the errors of the items within the group.
Then, a good ranking can be characterized as follows,
\begin{definition}[\epqbr]
\label{def:ours}
For numbers $p, q \in [1, \infty)$ and a non-negative weight function $w:[\gamma] \rightarrow \R_{\ge 0}$ that assigns weights to groups, $\bsigma \in \Sigma_{\cI}$ is an \epqbr~iff 
\begin{align}
    &\efair(\bsigma;\theta):=\bigg(\sum_{h \in [\gamma]}w(h) \cdot \err_h(\bsigma;\theta)^q\bigg)^{\frac{1}{q}} < \epsilon,\label{eq:epq_error}\\
&\text{where,}~~\forall h \in [\gamma],\err_h(\bsigma;\theta) := \paren{ \sum_{i \in G_h} d_i(\bsigma;\theta)^p}^{\frac{1}{p}}. \label{eq:group_wise_error}
\end{align}
\end{definition}
In our example in Figure~\ref{fig:example}, if we compute the $\ell_p$ norm of the error incurred by the items within the group for smaller values of $p$ (say $p = 1$), we get that group $B$ incurs an error of $3\epsilon$ and $2\epsilon$ with $\bsigma_1$ and $\bsigma_2$ respectively.
Whereas, group $A$ incurs an error of $0$ and $\epsilon$ in $\bsigma_1$ and $\bsigma_2$ respectively.
Further, $\ell_q$ norm of the errors across the groups, with $q = 1$ and $w(A) = w(B) = 1$, gives us that the overall $(\ell_p, \ell_q)$ error is $3\epsilon$ for both $\bsigma_1$ and $\bsigma_2$.
But the $(\ell_p, \ell_q)$ error for $\bsigma_3$ is still $\epsilon$.
Therefore, algorithms designed for our objective have $(\ell_p, \ell_q)$ error less than $\epsilon$, with high probability.

In this paper, we study the problem with the weight functions of the form $w(h) := \frac{\phi_h}{n_h}$ for each group $h \in [\gamma]$, where $1 \le \phi_h \le n_h$. This is already a rich class of weight functions since setting $\phi_h = 1$ measures the \textit{average} of the errors of the group's items, allowing us to achieve \textit{proportional errors} across groups. Setting $\phi_h = n_h$ counts the number of items from the group on which the algorithm makes an error, achieving \textit{equal errors} across groups.

Note that \epqbr~is a generalization of \ebr~as shown in the following theorem, the proof of which appears in Appendix~\ref{app:equivalence}.
\begin{restatable*}{thm}{equivalence}
\label{thm:equivalence}
A ranking $\bsigma\in \Sigma_{\cI}$ is an \epqbr~for $p, q \rightarrow \infty$ and any non-negative weight function $w$, if and only if it is an \ebr. 
\end{restatable*}

\begin{remark}
When $q \rightarrow \infty$, \Cref{def:ours} asks for the maximum group-wise $\ell_p$ norm error to be less than $\epsilon$.
This is similar to asking for egalitarian fairness while ranking items belonging to socially salient groups.
Such a notion has also been studied as socially fair clustering, first introduced in \cite{ghadiri21} and later generalized in a way similar to \Cref{def:ours} in \cite{chlamtac22}.
\end{remark}

A Probably-Approximately-Correct (PAC) ranker is a ranking algorithm that, for any problem instance including two parameters $\epsilon, \delta \in (0,1)$, sequentially makes a finite number of oracle calls and outputs a ranking $\bsigma$ of the items such that the error $err$ is less than $\epsilon$ with probability greater than $1-\delta$.
\cite{saha19a}
studied the PAC ranking problem in the Plackett-Luce feedback oracle setup with subset-wise preferences. 
They give an optimal sample complexity bound for error defined as \Cref{eq:theirs}.

\paragraph{Probably-Approximately-Correct and \textit{fair} ranker.}
We study a group-fair variant of the PAC ranking problem, where the only difference is that we measure the error using \Cref{eq:epq_error}. Then an algorithm that satisfies this can be~defined as follows,
\begin{definition}[\pacf]
A sequential algorithm that outputs a ranking is called \pacf~if it always outputs a ranking after a finite number of oracle calls, and the ranking output is an \epqbr~with probability at least $1-\delta$ for given parameters $p$, $q$ and a weight function $w$. 
\end{definition}
We refer to the number of oracle calls as \textit{sample complexity} interchangeably throughout this paper.
In conclusion, this paper aims to answer two questions:
\begin{enumerate}[leftmargin=*,noitemsep,topsep=0pt]
    \item \textit{What is the worst case minimum expected sample complexity required to learn an \epqbr?}
    \item \textit{Is there an \pacf~with matching sample complexity?}
\end{enumerate}

\begin{remark}(\textbf{Group-Aware vs.~Group-Blind}). Note that the group membership of the items may not always be available to the algorithm due to legal restrictions or simply because of the unavailability of the labeling. 
We call the algorithms with access to group information \aware~algorithms and those without access to group information \blind. In this paper, we answer the questions above for both types of algorithms.
\end{remark}

Further, we consider the class of algorithms that satisfy symmetry defined in \cite{saha19a}, additionally conditioning on appropriate mapping of the group membership. Roughly speaking, the algorithms should be insensitive to the specific labeling of the items. This property is needed to get tighter lower bounds. For a formal definition, see \Cref{def:symmetry} in \Cref{app:additional_notation}.

\section{Theoretical Results}
\label{sec:theoretical_results}
For the \blind~case, we show that using \algbtp~from \cite{saha19a} with appropriately adjusted error parameter $\epsilon$ and confidence parameter $\delta$ already gives us an efficient \pacf.
However, our key contribution is the lower bound on the sample complexity.
Formally, we prove the following theorem.
\begin{restatable*}[\textbf{\blind~sample complexity}]{thm}{blindbound}
\label{thm:upperbound_groupblind_multgroups}
Given an
error parameter $\epsilon \in (0,1)$, a confidence parameter $\delta \in [0,1]$, and a class of \blind~and symmetric  \pacf s, $\cA$, for PL feedback, there exists an instance $\nu$ such that any algorithm in $\cA$ on $\nu$ needs\\
$\Omega\paren{\frac{n^{1+\max\left\{\frac{2}{q},\frac{2}{p}\right\}}}{\epsilon^2}}$ 
samples.

\noindent Further, $\exists$ \pacf~with sample complexity 
$O\paren{\frac{n^{1+\max\left\{\frac{2}{q},\frac{2}{p}\right\}}}{\epsilon^2}\log\frac{n}{\delta}}$.
\end{restatable*}
Proving the lower bound on the sample complexity involves defining a true instance, i.e., the scores of the items, and defining a class of alternative instances with the scores modified from the true scores.
The crucial step here is to define an \textit{event} carefully so that the event being satisfied is a necessary condition for any \pacf~algorithm on the true instance.
Moreover, we need the complement of the event to be a necessary condition on the \pacf~algorithms for any of the alternative instances.

\begin{remark}
    The set of hard instances in \cite{saha19a} only give a loose lower bound of $\Omega\paren{\frac{n}{\epsilon^2}\log \frac{n}{\delta}}$.
    Hence, we need a more creative construction of the class of hard instances and the event, that takes into account that the errors of the items may accumulate in $\efair$ (\Cref{eq:epq_error}), rather than solely focusing on the maximum error, as done in $\err$ (\Cref{eq:theirs}).
\end{remark}
    
Below, we give a proof sketch where we focus on the design of our hard class of instances and the definition of a suitable event. The full proof can be found in \Cref{app:blind}.
\begin{proof}[Proof Sketch]
    
    For a given set of items $\cI$ with input parameters $\epsilon, \delta, p, q$, and $\phi_h, \forall h \in [\gamma]$, we fix a subset of $T \subset \cI$ of size $n/4$.
    Let $\teps = \epsilon\cdot\paren{\frac{4}{n}}^{1/p}$.
    Then, an instance of the problem is denoted by a subset $S \subseteq \cI\sm T$ such that the scores of the items for this instance are,
    \begin{gather*}
    \forall i \in S,~~\theta_i = \theta\paren{\frac{1}{2} + \teps}^2,\quad \forall i \in T,~~\theta_i = \theta\paren{\frac{1}{4}-\teps^2},\text{and}\quad\forall i \not\in S\cup T,~~\theta_i = \theta\paren{\frac{1}{2} - \teps}^2.
    \end{gather*}
    We then fix a set $\sS \subseteq \cI \sm T$ of size $|\sS| = n/4$ to be the true instance and any set $\tS^* = \sS \cup (\cI \sm (\sS\cup T))$ of size $|\tS^*| = n/2$ to be an alternative instance.
    Note that there will be ${{n/2} \choose {n/4}}$ such alternative instances for every true instance $\sS$.
    Then, we define the event to be,
    \[
    \cE(S) := \left\{\left|(S \cup T) \cap  \bsigma_{\cA}\left(\frac{n}{4}+2:n\right)\right| < \frac{n}{4}\right\}.
    \]
    That is, for any ranking $\bsigma_\cA$ output by any symmetric \pacf~on an instance $S$, the ranks $\frac{n}{4}+2$ to $n$ contain less than $\frac{n}{4}$ items from $\sS\cup T$.
    This is a high probability ($> 1-\delta$) event for $\sS$ because, otherwise, at least $n/4$ items suffer an error such that the errors add up to more than $\epsilon$.
    On the contrary, this event is a low probability ($< \delta$) event for alternative instances because unless all the items in $T$ appear in the ranks $\frac{n}{4}+2$ to $n$, the ranking can not have an error less than $\epsilon$.
    But $|T| = n/4$. Hence, if the event is satisfied, the error is greater than $\epsilon$. 
    Hence, we apply the change-of-measure inequality by \cite{Kaufmann2014OnTC} to get a lower bound tight up to $\log$ factors.\\
    \textbf{Upper bound.} Follows from the sample complexity guarantees of \algbtp~proved in \cite{saha19a}, because we run it with error parameter $\teps := \epsilon/n^{\max\{1/p,1/q\}}$ and confidence parameter $\delta$.
\end{proof}


Next, we show the sample complexity bounds for \aware~algorithms.
We design an algorithm that uses the additional ``group" information in an adaptive fashion.
The key idea in our algorithm design is ensuring that we efficiently balance our queries between inter-group and intra-group pairwise comparisons and in the right order.
Its sample complexity is as stated below.
\begin{restatable*}[\textbf{\aware~upper bound}]{thm}{awareupper}
\label{thm:upperbound_groupaware_multgroups}
    \Cref{alg:aware} is a \aware~\pacf\\with sample complexity $\cO\paren{\paren{\paren{\sum_{h = 1}^\gamma \frac{n_h^{1+\frac{2}{p}}}{\epsilon_h^2}}+\frac{n\cdot n_g^{2/p}\log\gamma}{\epsilon_{g}^2}}\log \frac{n}{\delta}}$, where $\epsilon_h = \epsilon\cdot\paren{\frac{n_h}{\phi_h \gamma}}^{1/q}$ and $g = \arg\min\limits_{h \in [\gamma]} \frac{\epsilon_h}{2}\cdot \paren{\frac{2}{n_h}}^{1/p}$.
\end{restatable*}
\begin{proof}[Proof Sketch] 
Briefly, Algorithm~\ref{alg:aware} proceeds in two steps:\\
\noindent\textbf{\textsf{Step 1)}} The algorithm finds group-wise rankings of the groups separately, by calling \algbtp~on each group with a group-dependent error parameter $\teps_h/2$ and a confidence parameter $\frac{\delta}{2n\gamma}$.
\algbtp~outputs a ranking of the items from the group such that between any two items from the group, the error is at most $\teps_h/2$.
This steps makes $\cO\paren{\paren{\sum_{h = 1}^\gamma n_h^{1+\frac{2}{p}}/\epsilon_h^2}\log\frac{n}{\delta}}$ queries.\\
\noindent\textbf{\textsf{Step 2)}} The algorithm then merges the group-wise rankings, two at a time as shown in the \textbf{while} loop in Lines 6 to 12.
The merging subroutine simply calls \algbtp~with error parameter corresponding to the lower $\teps_h$ amongst the two lists and confidence parameter $\frac{\delta}{2n\gamma}$ on pairs of items to get a pairwise winner. Using this, it merges the lists, similar to the \textit{merge} step in the merge sort algorithm.
Since we merge two lists at a time, after at most $\log\gamma$ many iterations, we will have one final sorted list, labeled as $\bsigma_1$.
Note that merging two lists of size $l_1$ and $l_2$ needs $O(l_1+l_2)T_q$ many queries, where $T_q$ is the pairwise query complexity.
For the error and confidence parameters mentioned above, \algbtp~returns pairwise ranking after $T_q = O\left(\frac{1}{\epsilon_{g}^2}\log \frac{n}{\delta}\right)$ many queries.
Therefore, each iteration of the $\log \gamma$ iterations of the \textbf{while} loop in Lines~6 to 12 makes $ O\left(\paren{(l_1+l_2)+(l_3+l_4)+\cdots+(l_{\gamma-1}+l_{\gamma})}/\epsilon_{g}^2\log \frac{n}{\delta}\right)$ $ = O\left(\frac{n}{\epsilon_{g}^2}\log \frac{n}{\delta}\right)$ many queries, concluding the proof.
\end{proof}
We complement our result with a lower bound, but for a restricted class of algorithms that only compare items from the same group (\textit{in-group} algorithms $\widetilde{\cA}$). 
\begin{restatable*}[\textbf{\aware~lower bound for $\widetilde{\cA}$}]{thm}{awarelower}
\label{thm:lowerbound_multgroups}
Given an error parameter $\epsilon \in (0,1)$, a confidence parameter $\delta \in [0,1]$, and a class of \aware, symmetric, and \textit{in-group} \pacf s, $\widetilde{\cA}$, for PL feedback, there exists an instance $\nu$ such that any algorithm in $\widetilde{\cA}$ on $\nu$ needs $\Omega\paren{\sum_{h = 1}^\gamma n_h^{1+\frac{2}{p}}/\epsilon_h^2}$ samples, where $\epsilon_h = \epsilon\cdot\paren{\frac{n_h}{\phi_h \gamma}}^{1/q}$.
\end{restatable*}

\begin{algorithm}[H]
\caption{Our ``\aware"~Algorithm}
\label{alg:aware}
\begin{algorithmic}[1]
    \INPUT{$\epsilon,\delta,p, q,\gamma, G_1, G_2, \ldots, G_{\gamma},  \phi_h, n_h, \forall h \in [\gamma]$.}
    \OUTPUT{\epqbr~of items $\bigcup_{h \in [\gamma]} G_h$.}
    \MAIN{Find-Ranking}
        \FOR{$h = 1,2,\ldots,\gamma$}\label{step:groupbtp_start}
        \STATE $\epsilon_h \gets \epsilon\cdot\paren{\frac{n_h}{\phi_h \gamma}}^{1/q}$, $\teps_h \gets \epsilon_h\cdot\paren{\frac{2}{n_h}}^{1/p}$
        \STATE $\bsigma_h \gets$ \algbtp~$\left(G_h, \frac{\teps_h}{2}, \frac{\delta}{2n\gamma}\right)$\label{step:groupbtp}
        \ENDFOR\label{step:groupbtp_end}
        \WHILE{$\gamma > 1$}\label{step:while_start}
            \FOR{$h = 1,3,5,\ldots,\gamma-1$}\label{step:loopstart}
                \STATE $\bsigma_h \gets$ \textsc{Merge}$\paren{\bsigma_h, \bsigma_{h+1}, \min\{\epsilon_h, \epsilon_{h'}\}, \frac{\delta}{2n^2\gamma}}$\label{step:merge}
                \STATE $\epsilon_h \gets \min\{\epsilon_h, \epsilon_{h'}\}$\label{step:eps_update}
            \ENDFOR
            \STATE $\gamma \gets \gamma/2$
        \ENDWHILE\label{step:while_end}
        \STATE \textbf{return} $\bsigma_1$
    \ENDMAIN

    \FUNCTION{\algbtp$\paren{S,\epsilon, \delta }$}
    \STATE Run \algbtp~from \cite{saha19a} on items in $S$ to get an \ebr.
    \ENDFUNCTION
    \FUNCTION{\textsc{Merge}$\paren{\bsigma, \bsigma', \epsilon, \delta}$}
        \STATE Set $\bsigma_{\text{merged}}$ to an empty ranking
        \WHILE{neither $\bsigma$ nor $\bsigma'$ is empty}\label{step:merge_while_start}
            \STATE $i \gets \pop(\bsigma)$, $i' \gets \pop(\bsigma')$
            \STATE $\bsigma_{\text{pair}} \gets$ \algbtp$(\{i,i'\},\epsilon,\delta)$\label{step:pairwise_ordering}
            \IF{$i$ is ranked lower in $\bsigma_{\text{pair}}$}
                \STATE Append $i$ to $\bsigma_{\text{merged}}$ 
            \ELSE 
                \STATE Append $i'$ to $\bsigma_{\text{merged}}$
            \ENDIF
        \ENDWHILE\label{step:merge_while_end}
        \IF{$\bsigma$ is empty}
            \STATE Append the rest of $\bsigma'$ to $\bsigma_{\text{merged}}$
        \ELSIF{$\bsigma'$ is empty}
            \STATE Append the rest of $\bsigma$ to $\bsigma_{\text{merged}}$
        \ENDIF
        \STATE \textbf{return} $\bsigma_{\text{merged}}$
    \ENDFUNCTION
\end{algorithmic}
\end{algorithm}

\begin{proof}[Proof Sketch]
    The set of instances we construct are those where finding a group-wise ranking is hard.
    Since the error metric $\efair$ aggregates errors across groups, we now define the event that depends on errors from at least half of the groups.
    Showing that such an event is good enough to differentiate sufficiently the true instance from the alternative instances is crucial, after which the proof follows from the lower bound proof for the \blind~case applied on the groups separately (see \Cref{app:aware}).
\end{proof}

\paragraph{Addressing the restriction on algorithms.}
For the algorithms that are allowed to make pairwise comparisons of items from different groups, it becomes challenging to bound the KL divergence between true and alternative instances for some of the pairwise comparisons.  
Hence, techniques other than using the change-of-measure argument by \cite{Kaufmann2014OnTC} may be needed to prove lower bounds for the entire class of \aware~and symmetric \pacf s.
\paragraph{Addressing the gap.}
The gap in our lower and upper bound for the \aware~case is mainly in terms of the sample complexity to merge the ranked lists of groups. We believe that our upper bound is optimal because it holds true under exact pairwise comparisons\footnote{see \href{https://cs.stackexchange.com/questions/144433/lower-bound-on-comparison-based-sorting-of-k-sorted-arrays}{this} for a short proof on the lower bound on pairwise comparisons to merge $k$ sorted lists} rather than comparisons drawn from the PL model.
However, a little thought will convince the reader that the techniques used to prove lower bound for the exact comparisons case do not readily extend to the stochastic feedback case, which is also seen in the problem of finding a sorted list of $n$ items. In the exact case, the sample complexity is $\Omega(n\log n)$. However, for the stochastic feedback case in the PAC learning setup studied in \cite{saha19a}, the lower bound is $\Omega\paren{\frac{n}{\epsilon^2}\log \frac{n}{\delta}}$, which needed several non-trivial ideas that diverge from the approaches employed in the exact sorting case.

We would also like to stress that the sample complexity needed to find group-wise rankings dominates the sample complexity to merge sorted group-wise rankings in some problem instances, ignoring the $\log \gamma$ factor.
For example, let $\phi_h = 1$ and $q > p$. Then, $\epsilon_h = c_1n_h^{1/q}$ and $\teps_h = c_2n_h^{1/q - 1/p}$ for some constants $c_1$ and $c_2$. Since $1/p > 1/q$, $\teps_h$ is inversely proportional to the size of the group, whereas $\epsilon_h$ is directly proportional to the size of the group.
W.l.o.g., let $n_1 \ge n_2 \ge \cdots \ge n_\gamma$.
Then, $g = \gamma$, and hence, $\epsilon_h \le \epsilon_g, \forall h\in[\gamma]$.
Therefore, the first term in the sample complexity is,
\begin{align*}
    \sum_{h \in [\gamma]}\frac{n_h^{1+\frac{2}{p}}}{\epsilon_h^2} \ge \sum_{h \in [\gamma]}\frac{n_h\cdot n_g^{\frac{2}{p}}}{\epsilon_g^2} = \frac{n\cdot n_g^{\frac{2}{p}}}{\epsilon_g^2},
\end{align*}
which is the second term without the $\log \gamma$ factor.
Therefore, our lower bound is reasonable as it is tight up to $\log$ factors for many parameter regimes.
\section{Experimental Results}
\label{sec:experimental_results}

In this section, we present an empirical analysis of our algorithms. 
\Cref{thm:upperbound_groupblind_multgroups} gives us that \algbtp~has almost optimal sample complexity as an \pacf. Hence, we use it as our \blind~baseline.
We use \Cref{alg:aware} as the \aware~algorithm.
We observe that \aware~almost always has strictly lower sample complexity on both real-world and synthetic datasets than \blind.
 
We use the datasets where the true scores of the candidates are available and use these scores to implement the Plackett-Luce sampling. 
For clarity, we kept few experimental results here and moved some of the plots on the real-world datasets, and all the plots on the synthetic datasets to \Cref{app:experiments}.
The experiments were run on an Intel(R)
Xeon(R) Silver 4110 CPU (8 cores, 2.1 GHz, and DRAM of 128GB). 

\begin{figure}[t]
    \centering
    \includegraphics[scale=0.22]{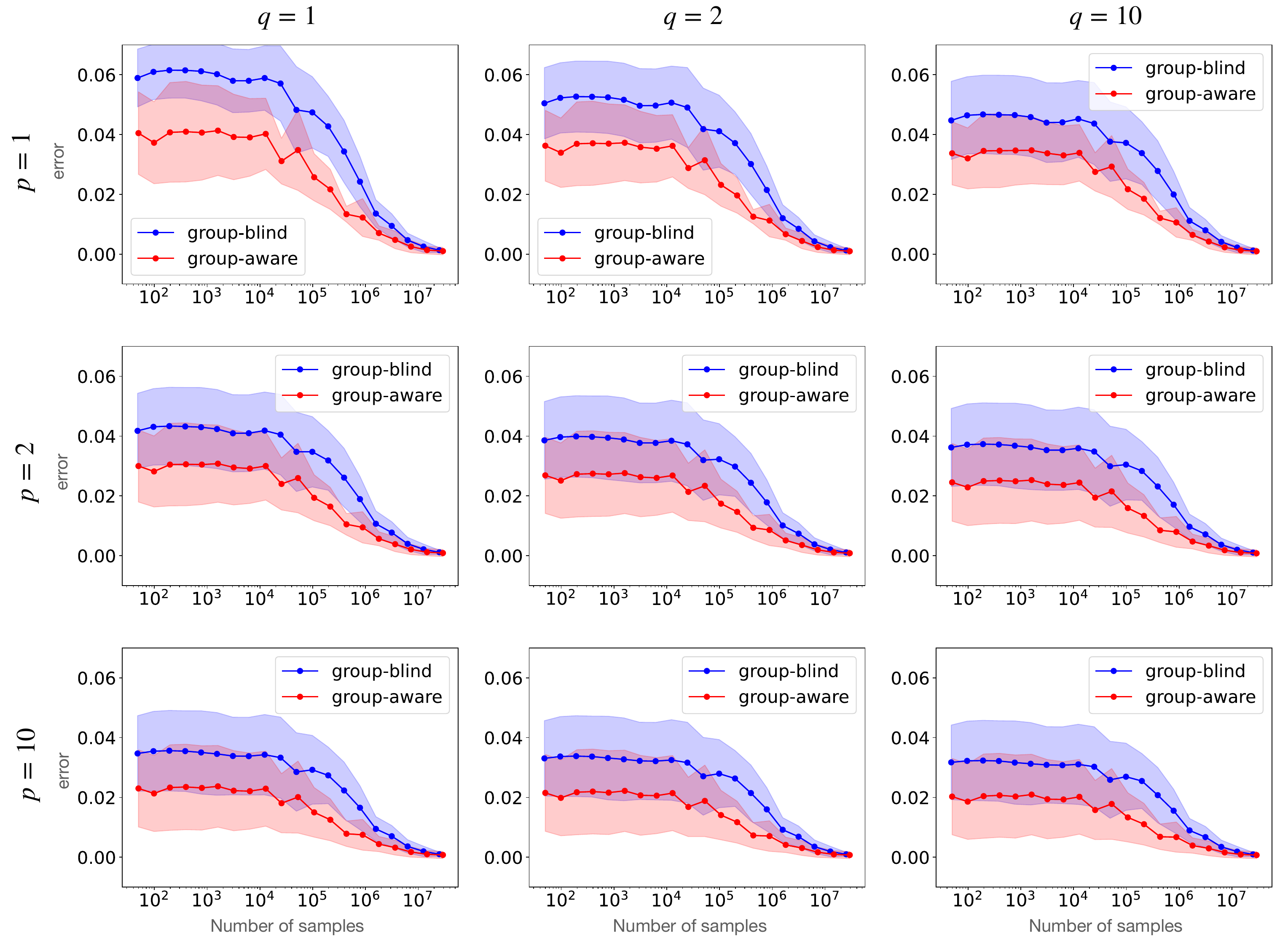}
    \caption{Group-Aware Ranking on German Credit with \textit{Age} defining two groups \textit{age $<25$} (minority) and \textit{age $\ge 25$}.}
    \label{fig:german_age25}
\end{figure}
\begin{figure}[t]
    \centering
    \includegraphics[scale=0.22]{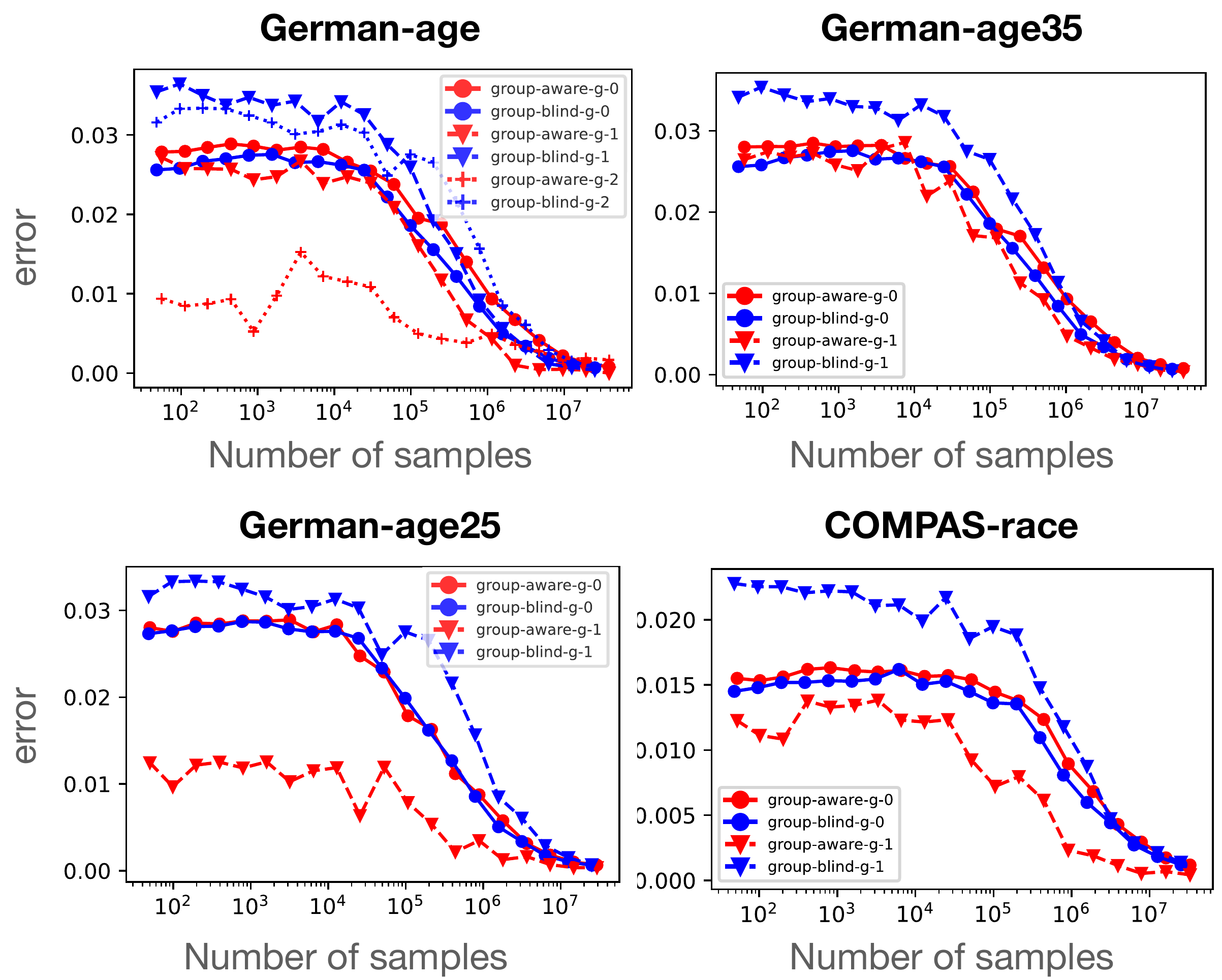}
    \caption{Group-wise errors (for $n=25, p=q=1$) for $g-0$ (majority group) and $g-1$, $g-2$ (minority groups).}
    \label{fig:group_errors}
\end{figure}

\paragraph{Real-World Datasets.}
\textbf{\textsf{(R1) COMPAS}.} It has been shown that the COMPAS tool disproportionately predicts higher recidivism scores for African-American defendants compared to others \citep{machinebias}. This leads to lower representation of African-Americans in the ranking based on \textit{$\neg$recidivism score} in the top few ranks.
Since the size of the groups affects the accuracy of the ranking for the groups, we use \textit{race} as the protected attribute and run experiments for two groups based on African-American or not (COMPAS-race).
It is also observed that the \textit{$\neg$recidivism score} is biased based on gender. Therefore, we ran experiments for two groups based on gender (COMPAS-gender). 
\\
\noindent \textbf{\textsf{(R2) German Credit.}} In this dataset taken from \cite{Dua:2019}, German adults are assigned a Schufa score indicating their creditworthiness, which have been observed to be discriminative towards younger adults (those of age $< 25$) \citep{cas19}.
Therefore, the resultant ranking based on these scores is also expected to be biased towards young adults hence reducing their representation in the top few ranks.
We run experiments with \textit{(i)} two groups based on \textit{age} split at $25$, \textit{(ii)} two groups based on \textit{age} split at $35$, and \textit{(iii)} three groups based on \textit{age} split at $25$ and $35$.
We call these datasets German-age25, German-age35, and German-age, respectively.
The exact proportions of the items according to the true scores are shown in \Cref{tab:proportions}; \Cref{app:experiments}.\\
\textbf{Reading the plots.}
In \Cref{fig:german_age25}, a point in the plot $(x,y)$ denotes the overall error $y$ of the ranking output by the algorithm after making $x$ many oracle queries, where the error is as defined in \Cref{def:ours}. In \Cref{fig:group_errors}, a point in the plot $(x,y)$ represents the group-wise error $y$ for a particular group, for the ranking output by the algorithm after making $x$ many oracle queries, where the group-wise error is as defined in \Cref{eq:group_wise_error}.
We show the mean and one standard deviation of $20$ runs of each algorithm.
The results are shown for $\phi_h = 1$ setting, however, we observe similar trends for $\phi_h = n_h$ case (see \Cref{fig:syn_unweighted,fig:syn_group_errors_unweighted} in \Cref{app:experiments}.


\begin{figure*}[t]
    \centering
    \includegraphics[width=0.9\textwidth]{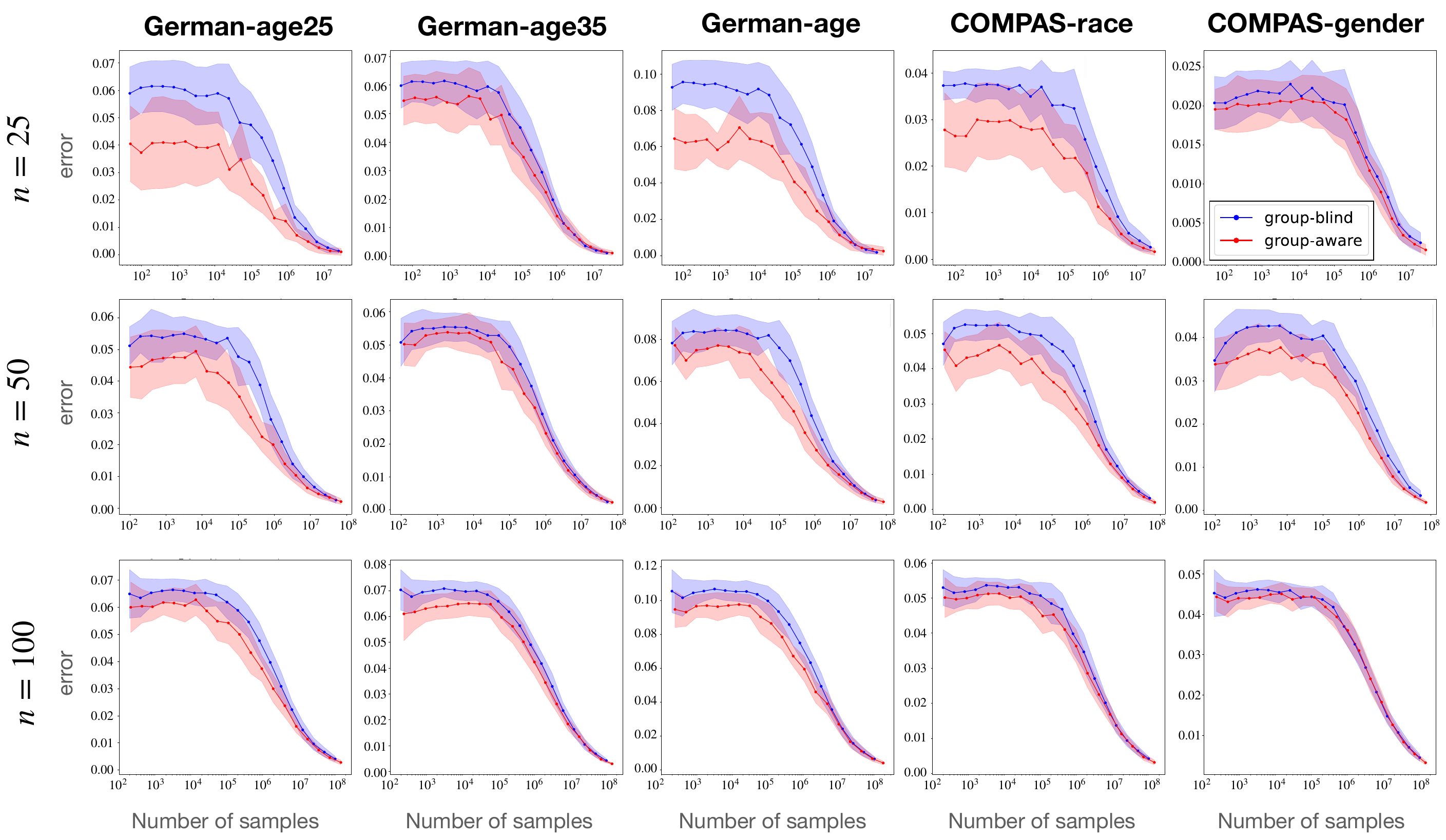}
    \caption{Experiments on the real-world datasets for different values of $n$ (for $p = q = 1$).}
    \label{fig:real_size}
\end{figure*}

We list our key observations below:
\paragraph{1) Sample Complexity with different values of $p$ and $q$.}
Our experimental results clearly show that the \aware~algorithm has lower sample complexity than the \blind~one, for both lower and higher values of $p$ and $q$ as seen in \Cref{fig:german_age25}. Notably, with smaller values of $p$ and $q$, the gap between the sample complexities for \blind~and \aware~algorithms  is significantly large compared to the higher values of $p$ and $q$, because in the former case, say $p = q = 1$, the error for each point gets counted in the overall error whereas for higher values, only the top few errors across the points within the group (top 1 as $p \rightarrow \infty$ case) and top few group-wise errors (top 1 as $q \rightarrow \infty$ case) count.
Since in this case \epqbr~is nothing but the \ebr~(as shown in \Cref{thm:equivalence}), both \blind~and \aware~algorithms have almost similar sample complexities.\\
\textbf{2) Group-wise errors.}
We also plot for each group, $\ell_p$ norm of the errors of the items in the groups appropriately normalized by the size of the group (see \Cref{fig:group_errors}).
We see a clear trend that for any group in any of these datasets, the \aware~algorithm has almost the same or smaller sample complexity than the \blind~one. Especially for minority groups, \aware~achieves much less error with fewer samples. This is because for a fixed number of queries, \aware \\surely samples some labels for items in the minority group while finding the group-wise rankings. However, a \blind~algorithm may end up making very few queries for items from the minority groups (especially if they are much smaller in proportion in the dataset) and, hence, has a very high error.\\
\textbf{3) Effect of $n$.}
The number of items we want to rank significantly changes the dynamics of the algorithms since the number of groups, their proportions and score distributions in the top $n$ ranking change significantly with $n$. 
\Cref{fig:real_size} shows these variations on several real-world datasets. On all of them, we observe that the \aware~has smaller sample complexity than \blind, with the gap more apparent for smaller values $p$ and $q$. 
This is consistent with Observation 1.

\section{Conclusion}
We study the Probably Approximately Correct (PAC) version of the problem of adaptively fair ranking $n$ items from pairwise comparisons in the Plackett-Luce (PL) preference model. We propose a fair metric for measuring the quality of rankings for different groups that generalizes ranking metrics that do not consider group fairness requirements. 
We study the problem under two settings: (i) where the ranking algorithm has access to group membership of items (\aware), and (ii) where the ranking algorithm does not have access to group membership of items (\blind). For the first setting, we show how the algorithm of \citet{saha19a} can be adjusted to find a fair ranking with optimal sample complexity, and we prove a matching lower bound on the sample complexity up to some $\log$ factors.
For the second setting, we design an algorithm and prove its sample complexity.
We also provide a reasonable lower bound for a restricted class of algorithms.

The main open question is to close the gap between the lower and upper bounds for both types of algorithms. 
It would also be interesting to study the problem under alternative choice models, such as the multinomial probit, Mallows, nested logit, generalized extreme-value models, etc.




\section*{Acknowledgements}
SG was supported by a Goolge PhD Fellowship.

\bibliographystyle{apalike}
\bibliography{refs}

\appendix
\onecolumn

\section{Missing Proof from \Cref{sec:prelims}}

\label{app:equivalence}

\equivalence
\begin{proof}
Let $d^{(h)}(\sigma; \theta) := \max_{i \in G_h}d_i(\sigma; \theta)$.
When $p \rightarrow \infty$ we have that 
\begin{align*}
\lim_{p \rightarrow \infty}\err_{h}(\sigma; \theta) &= \lim_{p \rightarrow \infty}\paren{\sum_{i \in G_h}d_i(\sigma; \theta)^p}^{1/p}\\
&= \lim_{p \rightarrow \infty}d^{(h)}(\sigma; \theta)\paren{\sum_{i \in G_h}\paren{\frac{d_i(\sigma; \theta)}{d^{(h)}(\sigma; \theta)}}^p}^{1/p}\\
&= d^{(h)}(\sigma; \theta)\lim_{p \rightarrow \infty}\paren{\sum_{i \in G_h}\paren{\frac{d_i(\sigma; \theta)}{d^{(h)}(\sigma; \theta)}}^p}^{1/p}.
\end{align*}
Notice that for $\frac{d_i(\sigma; \theta)}{d^{(h)}(\sigma; \theta)} \le 1, \forall i \in G_h$, the equality occurs at least once for one point and at most for all the points in $G_h$.
Since $p > 0$ and $d_i(\sigma; \theta) \ge 0, \forall i \in G_h$, we have that 
\begin{align*}
    1 \le \sum_{i \in G_h} \paren{\frac{d_i(\sigma; \theta)}{d^{(h)}(\sigma; \theta)}}^p \le n_h.
\end{align*}
Therefore,
\begin{align*}
    \lim_{p \rightarrow \infty}1^{1/p} &\le \lim_{p \rightarrow \infty}\paren{\sum_{i \in G_h} \paren{\frac{d_i(\sigma; \theta)}{d^{(h)}(\sigma; \theta)}}^p}^{1/p} \le \lim_{p \rightarrow \infty}n_h^{1/p}.
\end{align*}
But $\lim_{p \rightarrow \infty}1^{1/p} = 1$ and $\lim_{p \rightarrow \infty}n_h^{1/p} = 1$, which gives us that,
\begin{gather*}
    \lim_{p \rightarrow \infty}\paren{\sum_{i \in G_h} \paren{\frac{d_i(\sigma; \theta)}{d^{(h)}(\sigma; \theta)}}^p}^{1/p} = 1.\\
\end{gather*}
Therefore,
\begin{align}
\lim_{p \rightarrow \infty}\err_h(\sigma; \theta) &= d^{(h)}(\sigma; \theta) = \max_{i \in G_h}\paren{\max_{\substack{j \in [n]~\text{s.t.}\\ \theta_i > \theta_j \land \sigma(i) > \sigma(j)}} \theta_i - \theta_j}.
\end{align}

Let $\err_{h^*}(\sigma; \theta) := \max_{h \in [\gamma]}\err_h(\sigma; \theta)$.
When $q \rightarrow \infty$ we have that 
\begin{align*}
\lim_{q \rightarrow \infty}\efair(\sigma; \theta) &= \lim_{q \rightarrow \infty}\paren{\sum_{h \in [\gamma]}w(h)\cdot \err_h(\sigma; \theta)^q}^{1/q}\\
&= \lim_{q \rightarrow \infty}\err_h(\sigma; \theta)\paren{\sum_{h \in [\gamma]}w(h)\cdot\paren{\frac{\err_h(\sigma; \theta)}{\err_{h^*}(\sigma; \theta)}}^q}^{1/q}\\
&= \err_{h}(\sigma; \theta)\lim_{q \rightarrow \infty}\paren{\sum_{h \in [\gamma]}w(h)\cdot\paren{\frac{\err_{h}(\sigma; \theta)}{\err_{h^*}(\sigma; \theta)}}^q}^{1/q}.
\end{align*}
Notice that for $\frac{\err_{h}(\sigma; \theta)}{\err_{h^*}(\sigma; \theta)} \le 1, \forall h \in [\gamma]$, the equality occurs at least once for one group and at most for all the groups.
Since $q > 0$, $w(\cdot) \ge 0$, and $\err_{h}(\sigma; \theta) \ge 0, \forall h \in [\gamma]$, we have that 
\begin{align*}
    \min_{h \in [\gamma]} w(h) \le \sum_{h \in [\gamma]} w(h)\cdot \paren{\frac{\err_{h}(\sigma; \theta)}{\err_{h^*}(\sigma; \theta)}}^q \le \sum_{h \in [\gamma]} w(h).
\end{align*}
Therefore,
\begin{align*}
    \lim_{q \rightarrow \infty}\paren{\min_{h \in [\gamma]} w(h)}^{1/q} &\le \lim_{q \rightarrow \infty}\paren{\sum_{h \in [\gamma]} w(h)\cdot \paren{\frac{\err_{h}(\sigma; \theta)}{\err_{h^*}(\sigma; \theta)}}^q}^{1/q} \le \lim_{q \rightarrow \infty}\paren{\sum_{h \in [\gamma]} w(h)}^{1/q}.
\end{align*}
But $\lim_{q \rightarrow \infty}\paren{\min_{h \in [\gamma]} w(h)}^{1/q} = 1$ and $\lim_{q \rightarrow \infty}\paren{\sum_{h \in [\gamma]} w(h)}^{1/q} = 1$, which gives us that,
\begin{gather*}
    \lim_{q \rightarrow \infty}\paren{\sum_{h \in [\gamma]} w(h)\cdot \paren{\frac{\err_{h}(\sigma; \theta)}{\err_{h^*}(\sigma; \theta)}}^q}^{1/q} = 1.\\
\end{gather*}
Therefore,
\begin{align}
\lim_{q \rightarrow \infty}\efair(\sigma; \theta) &= \max_{h \in [\gamma]} \err_h(\sigma; \theta) = \max_{h \in [\gamma]}\max_{i \in G_h}\paren{\max_{\substack{j \in [n]~\text{s.t.}\\ \theta_i > \theta_j \land \sigma(i) > \sigma(j)}} \theta_i - \theta_j}.
\end{align}


Therefore, an $(\epsilon, \infty, \infty)$-Best-Ranking is the one that satisfies
\begin{align}
\max_{h \in [\gamma]}\quad&\max_{i \in G_h}\quad\paren{\max_{\substack{j \in [n]~\text{s.t.}\\ \theta_i > \theta_j \land \sigma(i) > \sigma(j)}} \theta_i - \theta_j} &< \epsilon\notag\\
\implies &\max_{i \in [n]}\quad\paren{\max_{\substack{j \in [n]~\text{s.t.}\\ \theta_i > \theta_j \land \sigma(i) > \sigma(j)}} \theta_i - \theta_j} &< \epsilon.\label{eq:pqbest}
\end{align}
Therefore, an $(\epsilon, \infty, \infty)$-Best-Ranking is any ranking that satisfies \Cref{eq:pqbest}.
Let $\sigma \in \Sigma_{[n]}$ be a ranking satisfying \Cref{eq:pqbest}.
Let $i, j \in [n]$ be a pair of distinct items such that $\theta_i \ge \theta_j + \epsilon$.
Then, $\sigma(i) < \sigma(j)$.
Therefore, $\not\exists i,j \in [n]$ such that $\sigma(i) > \sigma(j)$ and $\theta_i \ge \theta_j + \epsilon$.
Therefore, $\sigma$ is also an $\epsilon$-Best-Ranking according to \Cref{eq:theirs}.

Now let $\sigma \in \Sigma_{[n]}$ be an $\epsilon$-Best-Ranking according to \Cref{eq:theirs}.
Fix any $i \in [n]$.
Then for any $j \in [n]$ such that $\sigma(i) > \sigma(j)$, $\theta_i < \theta_j + \epsilon$.
Therefore, for the item $i$, 
\[\max_{\substack{j \in [n]~\text{s.t.}\\ \theta_i > \theta_j \land \sigma(i) > \sigma(j)}} \theta_i - \theta_j < \epsilon.\]
Therefore, $\sigma$ is also an $(\epsilon, \infty, \infty)$-Best-Ranking.
\end{proof}
\section{Missing Proof from \Cref{sec:theoretical_results}}
\subsection{Additional notation}
\label{app:additional_notation}
We use $\mathbb{I}\{\cE\}$ to denote the indicator function of the event $\cE$.
For any $a, b \in [0,1]$, $Ber(a)$ and $Geo(a)$ represent the Bernoulli and the Geometric distributions respectively, and $kl(a,b)$ represents the Kullback Leibler divergence between $Ber(a)$ and $Ber(b)$.
\begin{definition}[Symmetric Algorithm]
\label{def:symmetry}
     A group-blind PAC algorithm $\cA$ is said to be symmetric if its output is insensitive to the specific labeling of items, i.e., if for any PL model $(\theta_1, \ldots, \theta_n)$, with group memberships $(g_1, \ldots, g_n)$, bijection $\phi : [n] \rightarrow [n]$ and ranking $\bsigma : [n] \rightarrow [n]$, it holds that 
     \begin{multline}
         Pr(\cA \text{ outputs } \bsigma \mid (\theta_1, \ldots, \theta_n), (g_1, g_2, \ldots, g_n)) \\= Pr(\cA \text{ outputs } \bsigma \circ \phi\mid ((\theta_{\phi^{-1}(1)},  \ldots, \theta_{\phi^{-1}(n)}), (g_{\phi^{-1}(1)}, \ldots, g_{\phi^{-1}(n)})),
     \end{multline}where $Pr(\cdot \mid (\alpha_1, \ldots, \alpha_n))$ denotes the probability distribution on the trajectory of $\cA$ induced by the PL model $(\alpha_1, \ldots, \alpha_n)$.
\end{definition}

\subsection{Sample Complexity Bounds for \blind~Algorithms}
\label{app:blind}
\subsubsection{Upper bound}
\label{sec:blind_ubound}

Below we give our group-blind algorithm that outputs an \epqbr~with optimal query complexity.
\begin{algorithm}
\caption{Group-blind algorithm for single group}\label{alg:blind}
Run \algbtp~with error parameter $\teps = \frac{\epsilon}{n^{\max\left\{\frac{1}{p},\frac{1}{q}\right\}}}$ and confidence parameter $\delta$.\;
\end{algorithm}
\begin{lemma}\label{thm:blind_ubound}
    \Cref{alg:blind} is an \pacf~with sample complexity $\cO\paren{\frac{n^{1+\max\left\{\frac{2}{p},\frac{2}{q}\right\}}}{\epsilon^2}\log\frac{n}{\delta}}$.
\end{lemma}

\begin{proof}
Let $\bsigma$ be the ranking returned by \Cref{alg:blind}. From Theorem 8 in \cite{saha19a} we know that $\bsigma$ is an $\teps$-Best-Ranking with probability at least $1-\delta$. 
Using this, correctness of \Cref{alg:blind} can be shown as follows,
\begin{align}
\efair(\bsigma;\theta) 
&= \paren{\sum_{h\in[\gamma]}\frac{\phi_h}{|G_h|}\err_{h,p}{(\bsigma;\theta)^q}}^{1/q}
= \paren{\sum_{h\in[\gamma]}\frac{\phi_h}{|G_h|}\paren{\sum_{i \in G_h} \teps^p}^{q/p}}^{1/q}\notag\\
&= \paren{\sum_{h\in[\gamma]}\frac{\phi_h}{|G_h|}\paren{|G_h|^{q/p}\teps^q}}^{1/q}
= \teps\paren{\sum_{h\in[\gamma]}\frac{\phi_h}{|G_h|^{1-q/p}}}^{1/q}\notag\\
&\le \teps\paren{\sum_{h\in[\gamma]}|G_h|^{q/p} }^{1/q}&\because \phi_h \le |G_h|\notag\\
&= \frac{\epsilon}{n^{\max\{1/q,1/p\}}}\paren{\sum_{h\in[\gamma]}|G_h|^{q/p} }^{1/q}.\label{eq:beforecases}
\end{align}
When $q \le p$, we have that $1/q \ge 1/p$ and $|G_h|^{q/p}\le |G_h|$.
Therefore,
\begin{align}
    \efair(\bsigma;\theta) &\le \frac{\epsilon}{n^{1/q}}\paren{\sum_{h\in[\gamma]}|G_h| }^{1/q} \le \frac{\epsilon}{n^{1/q}}\cdot n^{1/q} = \epsilon.\label{eq:err_qlep}
\end{align}
When $q > p$, we have that $1/q < 1/p$. We need to upper bound the term in the summation in \Cref{eq:beforecases}.
Since the number of groups can be between $1$ and $n$, and the total number of items needs to be exactly $n$, we can write the following optimization problem where $x_h$ represents the number of items from group $h$,
\begin{align}
    \max_{x\in \R^{n}}~~~&\paren{\sum_{h \in [n]}x_h^{q/p}}^{1/q}\label{prog:1}\\
    \text{such that}~~~&\sum_{h \in [n]}x_h = n\label{prog:2}\\
    \text{and}~~~&x_h \in \{0,1,2,\ldots,n\},~~\forall h \in [n]\label{prog:3}.
\end{align}
Consider the following relaxation of the program above,
\begin{align}
    \max_{x\in \R^{n}}~~~&\paren{\sum_{h \in [n]}x_h^{q/p}}^{1/q}\label{eq:objective}\\
    \text{such that}~~~&\sum_{h \in [n]}x_h = n\label{eq:cons1}\\
    \text{and}~~~&0 \le x_h \le n,~~\forall h \in [n]\label{eq:cons2}.
\end{align}
It is easy to see that the all ones vector, $\mathbbm{1}$, is the minimizer of $\sum_{i \in [n]}x_i^{q/p}$ inside the convex polytope formed by the constraints (\ref{eq:cons1}) and (\ref{eq:cons2}). 
Therefore, the maximum value of this convex function is achieved at a vertex of the convex polytope.
Since the polytope lives in an $n$ dimensional space, any vertex is formed by at least $n$ equality constraints, one of which has to be (\ref{eq:cons1}).
Amongst the constraints in (\ref{eq:cons2}), at least $n-1$ coordinates of $x$ have to be $0$, and exactly one has to be $n$.
There is no other way to satisfy at least $n$ constraints with equality.
Moreover, $(\cdot)^{1/q}$ is a non-decreasing function.
Hence, the maximum of the objective function in \Cref{eq:objective} also occurs at a vertex of the convex polytope.
Note that this is also a feasible point for the program defined by \Cref{prog:1,prog:2,prog:3}.
Therefore, the maximum value of the objective function is $n^{1/p}$, which gives us that,
\begin{align}
    \efair(\bsigma;\theta) &\le \frac{\epsilon}{n^{1/p}}\cdot n^{1/p} = \epsilon.\label{eq:err_qgp}
\end{align}
From \Cref{eq:err_qlep,eq:err_qgp} we can conclude that \Cref{alg:blind} returns an \epqbr~with probability at least $1-\delta$.

From Theorem 8 in \cite{saha19a}, we know that \algbtp~with error parameter $\epsilon$ and confidence parameter $\delta$ has sample complexity $\cO\paren{\frac{n}{\epsilon^2}\log\frac{n}{\delta}}$.
Therefore, the sample complexity of \Cref{alg:blind} is $\cO\paren{\frac{n}{\teps^2}\log\frac{n}{\delta}} = \cO\paren{\frac{n^{1+\max\left\{\frac{2}{q},\frac{2}{p}\right\}}}{\epsilon^2}\log\frac{n}{\delta}}$.


\end{proof}
\subsubsection{Lower bound}
\label{sec:blind_lbound}

\begin{lemma}[Lower bound on Sample Complexity]
\label{thm:blind_lbound}
Given an error parameter $\epsilon \in (0,1)$, a confidence parameter $\delta \in (0,1)$, and a symmetric \pac\, algorithm $\cA$ for WI feedback, there exists a PL instance $\nu$ such that the sample complexity of $\cA$ on $\nu$ is at least
$
\Omega\bigg( \frac{n^{1+\max\left\{\frac{2}{p}, \frac{2}{q}\right\}}}{\epsilon^2} \bigg).
$
\end{lemma}

Similarly to \citep{saha19a}, we construct a hard class of instances and use the Lemma on multi-armed bandits given by \cite{Kaufmann2014OnTC} that gives a change-of-measure argument to lower bound the sample complexity.
We restate the lemma below,
\begin{lemma}[Lemma $1$, \cite{Kaufmann2014OnTC}]
\label{lem:gar16}
Let $\eta$ and $\eta'$ be two bandit models for $N$ arms (assignments of reward distributions to arms), such that $\eta_i ~(\text{resp.} \,\eta'_i)$ is the reward distribution of any arm $i \in [N]$ under the bandit model $\eta ~(\text{resp.} \,\eta')$, and such that for all such arms $i$, $\eta_i$ and $\eta'_i$ are mutually absolutely continuous. Then for any almost-surely finite stopping time $\tau$ with respect to $(\cF_t)_t$,
\begin{align*}
\sum_{i = 1}^{N}\E_{\eta}[N_i(\tau)]KL(\eta_i,\eta_i') \ge \sup_{\cE \in \cF_\tau} kl(Pr_{\eta}(\cE),Pr_{\eta'}(\cE)),
\end{align*}
where $kl(x, y) := x \log(\frac{x}{y}) + (1-x) \log(\frac{1-x}{1-y})$ is the binary relative entropy, $N_i(\tau)$ denotes the number of times arm $i$ is played in $\tau$ rounds, and $Pr_{\eta}(\cE)$ and $Pr_{\eta'}(\cE)$ denote the probability of any event $\cE \in \cF_{\tau}$ under bandit models $\eta$ and $\eta'$, respectively.
\end{lemma}

\begin{proof}[Proof of \Cref{thm:blind_lbound}]
Let us assume that the items belong to exactly one group. This result will be useful in proving sample complexity for multiple-group case. 
Note that when there is only one group, both \blind~and \aware~algorithms should incur the same minimum sample complexity to learn a ranking.
Also, the parameter $q$ and the weight function do not matter. 
For simplicity, let us assume that $n$ is a multiple of $4$.
\paragraph{Class of instances.}
Let $\teps = \epsilon\cdot\paren{\frac{4}{n}}^{1/p}$. 
Let $n' := 3n/4$.
Let $T := \{n'+1,n'+2,\ldots, n\}$.
Now consider the class of instances $\bnu_{[m]}$ for any $m \in [n']$, where for any $S \subseteq [n']$ such that $|S| = m$, $\nu_S$ represents the instance where 
\[
\forall i \in S,~~\theta_i = \theta\paren{\frac{1}{2} + \teps}^2,\quad \forall i \in T,~~\theta_i = \theta\paren{\frac{1}{4}-\teps^2},\quad \text{and}\quad\forall i \not\in S\cup T,~~\theta_i = \theta\paren{\frac{1}{2} - \teps}^2.
\]

\begin{remark}
\label{rem:label}
Note that $S$ uniquely represents an instance $\nu_S\in\bnu_{[m]}$ for any fixed $m\in [n']$.
\end{remark}

\begin{lemma}
\label{lem:inst_not}
For any $\theta > \frac{1}{1-2\teps}$ and for any problem instance $\nu_{S} \in \bnu_{[n/4]}$, any \epqbr, say $\bsigma_{S}$, has to satisfy the following:
the number of items from $S \cup T$ in ranks $\frac{n}{4}+2$ to $n$ should be strictly less than $\frac{n}{4}$. 
\end{lemma}
\begin{proof}
    Note that when $\theta > \frac{1}{1-2\teps}$, for any $i \in S, i' \in T$ and $i'' \notin S \cup T$, 
    \[
    \theta_i - \theta_{i'} > \teps,~~\theta_{i'} - \theta_{i''} > \teps,~~\text{and}~~\theta_{i} - \theta_{i''} > \teps.
    \]
    Let us assume that there exists an \epqbr~for $\nu_S$, say $\bsigma_S$, such that the ranks $\frac{n}{4}+2$ to $n$ have $\ge \frac{n}{4}$ items from $S\cup T$.
    Then, $\le \frac{n}{4}$ items from $S \cup T$ are in the ranks $1$ to $\frac{n}{4}+1$, since $|S\cup T| = \frac{n}{2}$.
    So there will be at least one item from $[n']\sm S$ in ranks $1$ to $\frac{n}{4}+1$, which implies that at least $\frac{n}{4}$ items in $S \cup T$ incur an error $> \teps$.
    Therefore, the overall error will be,
    \begin{align*}
        \efair(\bsigma_S;\theta_S) &= \paren{\sum_{i\in[n]}d_i(\bsigma_S;\theta_S)^p}^{1/p} > \paren{\frac{n}{4}\cdot\teps^p}^{1/p}= \paren{\frac{n}{4}\cdot \paren{\epsilon\cdot \paren{\frac{4}{n}}^{1/p}}^p}^{1/p}\\
        &= \paren{\epsilon^p}^{1/p}
        = \epsilon,
    \end{align*}
    which contradicts our assumption that $\bsigma_S$ is an \epqbr. 
\end{proof}

\paragraph{The alternative instances.}
We now fix any set $S^* \subset [n']$ such that $|S^*|= \frac{n}{4} $. Lower bound on the sample complexity is now obtained by applying \Cref{lem:gar16} on a pair of instances $(\nu_{S^*}, \nu_{\tS^*})$, for all possible choices of $\tS^* = \sS \cup S'$, where $S' \subset [n']\sm \sS$ and $|S'| = \frac{n}{4}$.
Note that there will be $n/2 \choose n/4$ choices of $\tS^*$.
\paragraph{Describing the event.}
For any ranking $\bsigma \in \Sigma_n$, we denote by $\bsigma(r:r')$ the set of items in the ranking, $\bsigma$, from rank $r$ to rank $r'$, for any $1 \le r \le r' \le n$.
Consider the event $\cE$ for an instance $S$ that the algorithm $\cA$ outputs a ranking such that the number of items from $S\cup T$ in the ranks $n/4+2$ to $n$ is $< n/4$. That is,
\[
\cE(S) := \left\{\left|(S \cup T) \cap  \bsigma_{\cA}\left(\frac{n}{4}+2:n\right)\right| < \frac{n}{4}\right\}.
\]
This is a high probability event for $S^*$ because otherwise, the error will be more than $\epsilon$, from \Cref{lem:inst_not}.
On the contrary, for the alternative instances with $\tS^*$, this is a low probability event because all the items from $T$ have to appear in the ranks $n/4+2$ to $n$ or otherwise, $\ge n/4$ items from $\tS^*\cup T$ will be in ranks $n/4+2$ to $n$ and they all incur an error $> \teps$ due to an item from $T$ in the ranks $1$ to $n/4+1$; therefore the total error will be more than $\epsilon$.

It is easy to note that as $\cA$ is an \pac\,, obviously 
\begin{equation}Pr_{S^*}( \cE(\sS) ) > Pr_{S^*}(\bsigma_{\cA} \text{ is an \ebr}) > 1-\delta,\label{eq:high_prob}
\end{equation}
and
\begin{equation}
Pr_{\tS^*}( \cE(\sS) ) < \delta,\label{eq:low_prob}
\end{equation}
for any alternative instance $\tS^*$.

We can further tighten \Cref{eq:low_prob} using \emph{symmetric property} of $\cA$ as follows,
\begin{lemma}
\label{lem:lb_sym}
For any symmetric \pac\,, $\cA$, and any problem instance $\nu_S \in \bnu_{[n']}$ such that $|S| = \frac{n}{2}$,
$Pr_{S}\left( \cE(S)\right) < \frac{\delta}{{n/2 \choose n/4}},
$
where $Pr_{S}(\cdot)$ denotes the probability of an event under the underlying problem instance $\nu_S$ and the internal randomness of the algorithm $\cA$ (if any).
\end{lemma}
\begin{proof}
Let us first fix an $m = n/2$ and $m' = n/4$.
Consider a problem instance $\nu_S \in \bnu_{[m]}$. Recall from Remark \ref{rem:label} that we use the notation $S \in \bnu_{[m]}$ to denote a problem instance in $\bnu_{[m]}$. 
Then the probability of making an error over all possible choices of $S \in \bnu_{[m]}$:

\begin{equation}
\label{eq:lb_identity}
\sum_{S \in \bnu_{[m]} }Pr_{S}\Big( \cA~\text{makes an error $> \epsilon$ on }S\Big) \ge \sum_{S_1 \in \bnu_{[m']}}\sum_{\substack{S_2 \in [n']\sm S \\ \text{s.t.} |S_2| = \frac{n}{4}}}Pr_{S_1\cup S_2}\Big( \cE(S_1)) \Big)
\end{equation}
where the above analysis follows from a similar result proved by \cite{Kalyanakrishnan+12} to derive sample complexity lower bound for classical multi-armed bandit setting towards recovering top-$q$ items (see Theorem $8$, \cite{Kalyanakrishnan+12}).

Clearly the possible number of instances in $\bnu_{[m]}$, i.e. $|\bnu_{[m]}| = {{n'}\choose{m}}$, as any set $S \subset [n']$ of size $m$ can be chosen from $[n']$ in ${{n'}\choose{m}}$ ways.

Now from \emph{symmetry} of algorithm $\cA$ and by construction of the class of our problem instances $\bnu_{[m]}$, for any two instances $S_1$ and $\hat{S}_1$ in $\bnu_{[m']}$, and for any choices of $S_2 \in [n']\setminus S'_1$ and $\hat{S}_2 \in [n']\setminus S'_2$ such that $|\hat{S}_1| = |\hat{S}_2| = n/4$ we have that:

\begin{equation*}
Pr_{S_1\cup S_2}\Big( \cE(S_1) \Big) = Pr_{\hat{S}_1 \cup \hat{S}_2}\Big( \cE(\hat{S}_1) \Big).
\end{equation*}

Then the above equivalently implies that for all $S_1 \in \bnu_{[m']}$ and any $S_2 \in [n']\setminus S$ such that $|S_2| = \frac{n}{4}$, $\exists p \in [0,1]$
\[
Pr_{S_1\cup S_2}(\cE(S_1)) = p.
\]
Then using above in \Cref{eq:lb_identity} we can further derive,

\begin{align*}
    \sum_{S \in \bnu_{[m]} } Pr_{S}\Big( \cA~\text{makes an error $> \epsilon$ on }S\Big) & \ge \sum_{S_1 \in \bnu_{[m']}}\sum_{S_2 \in [n']\sm S_1:|S_2| = n/4}Pr_{S_1 \cup S_2}\Big( \cE(S_1)) \Big)\\
    &= {n'\choose m'}{n'-m' \choose n/4} p\\
    &= {3n/4\choose n/4}{n/2 \choose n/4}p.
\end{align*}

On the L.H.S.~there are ${3n/4 \choose n/2} = {3n/4 \choose n/4}$ choices of $S$.
Therefore, if $p > \frac{\delta}{{n/2 \choose n/4}}$, we get that,
\begin{align*}
    \sum_{S \in \bnu_{[m]} } Pr_{S}\Big( \cA~\text{makes an error $> \epsilon$ on }S\Big) & >  {\frac{3n}{4} \choose \frac{n}{4}}\delta.
\end{align*}
which in turn implies that there exists at least one instance $\nu_S \in \bnu_{[m]}$ such that \\$Pr_{S}\Big(  \cA~\text{makes an error $> \epsilon$ on }S \Big) \ge \delta$, which violates the \pac\ property of algorithm $\cA$. Thus it has to be the case that $p < \frac{\delta}{{n/2 \choose n/4}}$,
which concludes the proof.
\end{proof}

Owing to \Cref{lem:lb_sym} we get,

\begin{equation}
Pr_{\tS^*}\Big( \cE(\sS) \Big) < \frac{\delta}{{n/2 \choose n/4}}.\label{eq:lbthm_2}
\end{equation}

We will crucially use \Cref{eq:high_prob} and \Cref{eq:lbthm_2} in the following lemma. 
\begin{lemma}[Lemma 26, \cite{saha19a}]
\label{lem:kl_del}
For any $\delta \in (0,1)$, and $\alpha \in \R_+$,
$
kl\bigg(1-\delta,\dfrac{\delta}{\alpha}\bigg) > \ln \dfrac{\alpha}{4\delta}.
$ 
\end{lemma}

This lemma leads to the desired tighter upper bound for $kl(Pr_{\nu_{\sS}}(\cE),Pr_{\nu_{\tS^*}}(\cE)) \ge kl(1-\delta,\frac{\delta}{{n/2 \choose n/4}}) \geq \ln \frac{{n/2 \choose n/4}}{4\delta}$.

Note that for the problem instance $\nu_{S^*} \in \bnu_{[m]}$, the probability distribution associated with a particular arm $B \in \cB$ -- a pair of items -- is given by:
\[
\nu^B_\sS \sim Categorical(p_1, p_2), \text{ where } p_i = Pr(i|B), ~~\forall i \in [2], \, \forall B \in \cB,
\]
where $Pr(i|B)$ is the probability of item $i$ winning in the Plackett-Luce model for the items in set $B$. 
Now applying \Cref{lem:gar16}, for some event $\cE \in \cF_\tau$ we get,

\begin{align}
\label{eq:FI_a}
\sum_{\{B \in \cB\}}\E_{\nu^B_\sS}[N_B(\tau_{\cA})]KL(\nu^B_\sS, \nu^B_{\tS^*}) \ge kl(Pr_{\nu_\sS}(\cE),Pr_{\nu_{\tS^*}}(\cE)),
\end{align}
where $N_B(\tau_{\cA})$ denotes the number of times arm $B$ is played by $\cA$ in $\tau$ rounds.
Note that whenever $B$ is such that $B \subset ([n']\sm \tS^*) \cup \sS, KL(\nu^B_\sS, \nu^B_{\tS^*})=0$.
Therefore, we will only focus on $B$ such that $B \not\subset ([n']\sm \tS^*) \cup \sS$.

We simplify the right-hand side of \Cref{eq:FI_a} using the following lemma.
\begin{lemma}
\label{lem:kl_div}
    For any $m \in [n']$ and problem instance $\nu_S \in \bnu_{[m]}$, and any arm $B$, $KL(\nu^B_\sS, \nu^B_{\tS^*}) \le 64\teps^2$.
\end{lemma}
\begin{proof}
Let $B = \{a, b\}$, i.e., the pair of items $a, b$ for which we make the oracle call.
Let 
\[
R = \frac{\frac{1}{2}+\teps}{\frac{1}{2}-\teps}.
\]
\paragraph{Case 1: None of the items in $B$ are from $T$.}
In this case, we again have four cases.
\begin{enumerate}
    \item \textit{Both items are good:} $\nu_{\sS}^B(a) = \nu_{\sS}^B(b) = \frac{\theta\paren{\frac{1}{2}+\teps}^2}{\theta\paren{\frac{1}{2}+\teps}^2+\theta\paren{\frac{1}{2}+\teps}^2} = \frac{1}{2}$. 
    \item \textit{$a$ is good and $b$ is bad:} $\nu_{\sS}^B(a) = \frac{\theta\paren{\frac{1}{2}+\teps}^2}{\theta\paren{\frac{1}{2}+\teps}^2+\theta\paren{\frac{1}{2}-\teps}^2} = \frac{R^2}{R^2+1}$ and $\nu_{\sS}^B(b) = \frac{\theta\paren{\frac{1}{2}-\teps}^2}{\theta\paren{\frac{1}{2}-\teps}^2+\theta\paren{\frac{1}{2}+\teps}^2} = \frac{1}{R^2+1}$.
    \item \textit{$a$ is bad and $b$ is good:} 
    $\nu_{\sS}^B(a) = \frac{\theta\paren{\frac{1}{2}-\teps}^2}{\theta\paren{\frac{1}{2}-\teps}^2+\theta\paren{\frac{1}{2}+\teps}^2} = \frac{1}{R^2+1}$ and $\nu_{\sS}^B(b) = \frac{\theta\paren{\frac{1}{2}+\teps}^2}{\theta\paren{\frac{1}{2}+\teps}^2+\theta\paren{\frac{1}{2}-\teps}^2} = \frac{R^2}{R^2+1}$
    \item \textit{Both items are bad:} $\nu_{\sS}^B(a)  = \nu_{\sS}^B(b) = \frac{\theta\paren{\frac{1}{2}-\teps}^2}{\theta\paren{\frac{1}{2}-\teps}^2+\theta\paren{\frac{1}{2}-\teps}^2} = \frac{1}{2}$.
\end{enumerate}

Now we use the upper bound from \cite{klub16}, $KL(\p_1,\p_2) \le \sum_{x \in \X}\frac{p_1^2(x)}{p_2(x)} -1$, for two probability mass functions $\p_1$ and $\p_2$ on the discrete random variable $\X$.
In this case, the KL divergence is non-zero when there is one item from $([n']\sm \tS^*) \cup \sS$, say $a$, and one item from $\tS^* \sm \sS$, say $b$.
Let $\tS^*$ be such that both $a$ and $b$ are good for that alternative instance. We will have subcases:
\begin{enumerate}
    \item \textit{When $a \in \sS$:} \begin{align*}
    KL(\nu^B_\sS, \nu^B_{\tS^*}) &\le \paren{\frac{R^2}{R^2+1}}^2\cdot\frac{2}{1} + \paren{\frac{1}{R^2+1}}^2\cdot\frac{2}{1} - 1 \\
    &= \frac{2R^4+2-(R^2+1)^2}{(R^2+1)^2}\\
    &= \frac{2R^4+2-R^4-2R^2-1}{(R^2+1)^2}\\
    &= \frac{R^4-2R^2+1}{(R^2+1)^2}\\
    &= \frac{(R^2-1)^2}{(R^2+1)^2}\\
    &= \paren{\frac{\paren{\frac{1}{2}+\teps}^2 - \paren{\frac{1}{2}-\teps}^2}{\paren{\frac{1}{2}+\teps}^2 + \paren{\frac{1}{2}-\teps}^2}}^2\\
    &= \paren{\frac{2\teps}{\frac{1}{2}+2\teps^2}}^2\\
    &\le \paren{\frac{2\teps}{\frac{1}{2}}}^2\\
    &= 64\teps^2.
    \end{align*}
    \item \textit{When $a \in ([n']\sm \tS^*)$: }
    \begin{align*}
    KL(\nu^B_\sS, \nu^B_{\tS^*}) &\le \paren{\frac{1}{2}}^2\cdot\frac{R^2+1}{1} + \paren{\frac{1}{2}}^2\cdot\frac{R^2+1}{R^2} - 1 \\
    &= \frac{R^2+1}{4}+\frac{R^2+1}{4R^2}-1\\
    &= \frac{R^4+R^2+R^2+1-4R^2}{4R^2}\\
    &= \frac{R^4-2R^2+1}{4R^2}\\
    &= \frac{(R^2-1)^2}{4R^2}\\
    &= \frac{1}{4}\paren{\frac{R^2-1}{R}}^2\\
    &= \frac{1}{4}\paren{R-\frac{1}{R}}^2\\
    &\le 64\teps^2.
    \end{align*}
\end{enumerate}

\paragraph{Case 2: Exactly one item in $B$ is from $T$.}
Let $b \in T$.
Again, we have several cases depending on item $a$.
\begin{enumerate}
    \item \textit{$a\in \sS$:} $\nu_{\sS}^B(a) = \frac{\theta\paren{\frac{1}{2}+\teps}^2}{\theta\paren{\frac{1}{2}+\teps}^2+\theta\paren{\frac{1}{4}-\teps^2}} = \frac{R}{R+1}$ and
    $\nu_{\sS}^B(b) = \frac{\theta\paren{\frac{1}{4}-\teps^2}}{\theta\paren{\frac{1}{2}+\teps}^2+\theta\paren{\frac{1}{4}-\teps^2}} = \frac{1}{R+1}$.
    \item \textit{$a \in [n']\sm \sS$:} $\nu_{\sS}^B(a) = \frac{\theta\paren{\frac{1}{2}-\teps}^2}{\theta\paren{\frac{1}{2}-\teps}^2+\theta\paren{\frac{1}{4}-\teps^2}} = \frac{1}{R+1}$ and
    $\nu_{\sS}^B(b) = \frac{\theta\paren{\frac{1}{4}-\teps^2}}{\theta\paren{\frac{1}{2}-\teps}^2+\theta\paren{\frac{1}{4}-\teps^2}} = \frac{R}{R+1}$.
\end{enumerate}
Note that here we have only one case where $a \not\in T$ is bad in $\sS$ and good in $\tS^*$.
Then,
\begin{align*}
    KL(\nu^B_\sS, \nu^B_{\tS^*}) &\le \paren{\frac{1}{R+1}}^2\cdot\frac{R+1}{R}+\paren{\frac{R}{R+1}}^2\cdot\frac{R+1}{1}-1\\
    &= \frac{1}{R(R+1)}+\frac{R^2}{R+1}-1\\
    &= \frac{1+R^3-R^2-R}{R(R+1)}\\
    &= \frac{(1-R)^2(1+R)}{R(R+1)}\\
    &= \frac{(1-R)^2}{R}\\
    &= \paren{\frac{\frac{1}{2}-\teps - \frac{1}{2}-\teps}{\frac{1}{2}+\teps}}^2\cdot\paren{\frac{\frac{1}{2}+\teps}{\frac{1}{2}-\teps}}\\
    &= \frac{4\teps^2}{\frac{1}{4}-\teps^2}\\
    &\le 32\teps^2&\because~~\teps^2 > \frac{1}{8} \implies \frac{1}{4}-\teps^2 > \frac{1}{8}\\
    &\le 64\teps^2.
\end{align*}
\end{proof}
Now applying \Cref{lem:gar16} and \Cref{lem:kl_div} for each altered problem instance $\nu^B_{\tS^*}$, each corresponding to any one of the ${n/2 \choose n/4}$ different choices of $S' \in [n']\setminus \sS$ such that $|S'| = n/4$, and summing all the resulting inequalities gives:

\begin{align}
\label{eq:lb1}
\sum_{S' \in [n']\setminus\sS: |S'| = n/4}\sum_{B}\E_{\nu^B_\sS}[N_B(\tau_{\cA})]KL(\nu^B_\sS,\nu^B_{\tS^*}) \ge {\frac{n}{2}\choose \frac{n}{4}}\ln \frac{{n/2\choose n/4}}{4\delta}.
\end{align}

In the left-hand side of \Cref{eq:lb1} above, the arms $B$ that have KL divergence $>0$ are those where exactly one item in $B$ is flipped from bad to a good item in $\sS$ and $\tS^*$ respectively.
Therefore, such a $B$ shows up for exactly ${n/2-1 \choose n/4-1}$ many times. Thus, given a fixed set $B$, the coefficient of the term $\E_{\nu^B_\sS}$ becomes ${\frac{n}{2}-1\choose \frac{n}{4}-1}64\teps^2$.

Therefore,
\begin{align}
\label{eq:lb2}
\nonumber \sum_{\{B \in \cB\}} \E_{\nu^B_\sS}[N_B(\tau_A)]KL(\nu^B_\sS,\nu^B_{\tS^*}) \le \sum_{\{B \in \cB\}} \E_{\nu^B_\sS}[N_B(\tau_A)]{\frac{n}{2}-1\choose \frac{n}{4}-1}64\teps^2.
\end{align}

Finally noting that $\tau_A \ge \sum_{B \in \cB}[N_B(\tau_A)]$, we get 

\begin{align*}
{\frac{n}{2}-1\choose \frac{n}{4}-1}64\teps^2\E_{\nu^B_\sS}[\tau_A] =  \sum_{S \in \cB}\E_{\nu^B_\sS}[N_B(\tau_A)]({\frac{n}{2}-1\choose \frac{n}{4}-1}64\teps^2) \ge {\frac{n}{2}\choose \frac{n}{4}}\ln \frac{{\frac{n}{2}\choose \frac{n}{4}}}{4\delta}.
\end{align*}
Note that ${n \choose r} = {n-1\choose r-1}\frac{n}{r}$. 
Therefore,
\begin{align*}
    \E[\tau_A] \ge \frac{1}{\teps^2}\frac{n/2}{n/4}\ln \frac{{\frac{n}{2}\choose \frac{n}{4}}}{4\delta} = \frac{2}{\teps^2}\ln \frac{{\frac{n}{2}\choose \frac{n}{4}}}{\delta} \ge \frac{2}{\teps^2} \ln \frac{2^{n/2}\sqrt{6}}{\sqrt{\pi\paren{3n/2+2}}\delta} = \frac{n}{\teps^2} + \frac{1}{\teps^2}\ln\frac{6}{\pi\paren{3n/2+2}\delta} \ge \frac{n}{\teps^2}.
\end{align*}
When, $p < q$, we consider the instances to be such that all the items are from the same group.
Hence, the above gives the sample complexity.
Whereas, when $p \ge q$, we consider the instances to be such that each item is from a different group.
Even in this case, the objective function is similar to the one-group case with the parameter $q$ in place of $p$. 
Therefore, with the one-group case we are able to compute the lower bound on the sample complexity as $\Omega\left(\frac{n^{1+\max\left\{\frac{2}{p}, \frac{2}{q}\right\}}}{\epsilon^2}\right)$.
This is loose by a multiplicative factor of $\ln\frac{n}{\delta}$.
\end{proof}
\blindbound
\begin{proof}
    Follows from \Cref{thm:blind_lbound} and \Cref{thm:blind_ubound}.
\end{proof}
\subsection{Sample Complexity Bounds for \aware~Algorithms}
\label{app:aware}

\subsubsection{Upper Bound}
\label{sec:aware_ubound}
\awareupper
\begin{proof}
    Let $\bsigma$ be the ranking returned by \Cref{alg:aware}, and $\epsilon_h = \epsilon\cdot\paren{\frac{n_h}{\phi_h \gamma}}^{1/q}, \forall h \in [\gamma]$.
    We will first prove the correctness of \Cref{alg:aware}.
    \paragraph{Correctness.} \Cref{alg:aware} computes the ranking in two steps, (1) finding group-wise ranking using the \algbtp~algorithm and (2) merging the group-wise rankings to get an overall ranking. In both the steps, the algorithm ensures that between any pair of items from groups $h$ and $h'$, the error in their pairwise ranking is at most $\min\{\epsilon_h, \epsilon_{h'}\}$.
    We will now show that this ensures that the $\err_{p,q}(\bsigma;\theta) < \epsilon$. 
    For ease of exposition, let us define the following intra-group and inter-group errors for each item,
    \[
    d^{(h)}_i(\bsigma; \theta)  := \max_{\substack{i' \in G_h~\text{s.t.}\\ \theta_i > \theta_{i'} \land \bsigma^{-1}(i) > \bsigma^{-1}(i')}} \theta_i - \theta_{i'}\qquad\text{and}\qquad d^{(\neg h)}_i(\bsigma; \theta)  := \max_{\substack{i' \in [n]\sm G_h~\text{s.t.}\\ \theta_i > \theta_{i'} \land \bsigma^{-1}(i) > \bsigma^{-1}(i')}} \theta_i - \theta_{i'}.
    \]
     Fix a group $h \in [\gamma]$ and let $n_h = |G_h|$, the size of the group.
        Let $\teps_h = \frac{\epsilon_h}{2}\cdot \paren{\frac{2}{n_h}}^{1/p}$ and $\delta'_h = \frac{\delta}{2n\gamma}$.
        First, \algbtp~is run for the items within the group $h$.
        Therefore, from \Cref{thm:blind_ubound} we know that $\bsigma$ is an $\teps_h$-Best-Rank with probability at least $1-\delta'_h$. 
        Therefore, for any item $i \in G_h$,
        \[
        \Pr\left[d_i^{(h)}(\bsigma;\theta) > \teps_h\right] < \delta'_h = \frac{\delta}{2n\gamma}. 
        \]

        For any pair of items $i, i'$ from two different groups $h, h'$ respectively, their ordering will be decided using \algbtp~with error parameter $\min\{\teps_h, \teps_{h'}\}$ and confidence parameter $\frac{\delta}{2n^2\gamma}$.
        W.l.o.g.~let $\teps_h \le \teps_{h'}$.
        Then, 
        \[
        \Pr\left[d_i^{(\neg h)}(\bsigma;\theta) > \teps_h\right] < \frac{\delta}{2n^2\gamma}~~\text{and}~~\Pr\left[d_{i'}^{(\neg h')}(\bsigma;\theta) > \teps_{h'}\right] < \Pr\left[d_{i'}^{(\neg h')}(\bsigma;\theta) > \teps_{h}\right] < \frac{\delta}{2n^2\gamma}.
        \]
        
        Using this, correctness of \Cref{alg:aware} can be shown as follows,
        
        \begin{align*}
            \Pr[\efair(\bsigma;\theta) > \epsilon]&=\Pr\sparen{\paren{\sum_{h\in[\gamma]}\frac{\phi_h}{|G_h|}\err_{h,p}{(\bsigma;\theta)^q}}^{1/q} > \epsilon}\\
            &\le \Pr\sparen{\bigvee\limits_{h\in [\gamma]} \bigvee\limits_{i\in G_h} d_i(\bsigma;\theta) > \teps_h}.
        \end{align*}
        Because otherwise, if for all $h \in [\gamma]$, and $\forall i \in G_h$, $d_i(\bsigma;\theta) \le \epsilon_h$, we have that,
        \begin{align*}
            \err_h(\bsigma;\theta) &\le \paren{n_h \paren{\frac{\epsilon_h}{2}\paren{\frac{2}{n_h}}^{1/p}}^p}^{1/p}\\&= \paren{n_h\cdot \frac{\epsilon_h^p}{2^p}\cdot \frac{2}{n_h}}^{1/p}\\
            &= \epsilon_h \cdot \frac{2^{1/p}}{2} \le \epsilon_h &\because~~p \ge 1 \implies 1/p\le 1\\
            \implies \efair &\le \paren{\sum_{h \in [\gamma]}\frac{\phi_h}{n_h}\paren{\epsilon\cdot\paren{\frac{n_h}{\phi_h \gamma}}^{1/q}}^q}^{1/q}\\&= \paren{\sum_{h \in [\gamma]}\frac{\phi_h}{n_h}\cdot \epsilon_q \cdot \frac{n_h}{\phi_h \gamma}}^{1/q} \le \epsilon.
        \end{align*}
        Therefore,
        \begin{align*}
        \Pr[\efair&(\bsigma;\theta) > \epsilon] \\
            &\le \Pr\sparen{\bigvee\limits_{h\in [\gamma]} \bigvee\limits_{i\in G_h} \paren{ d_i^{(h)}(\bsigma;\theta) > \teps_h \lor d_i^{(\neg h)}(\bsigma;\theta) > \teps_h}}\\
            &= \Pr\sparen{\bigvee\limits_{h\in [\gamma]} \bigvee\limits_{i\in G_h}  d_i^{(h)}(\bsigma;\theta) > \teps_h} + \Pr\sparen{\bigvee\limits_{h\in [\gamma]} \bigvee\limits_{i\in G_h}  d_i^{(\neg h)}(\bsigma;\theta) > \teps_h}\\
            &= \Pr\sparen{\bigvee\limits_{h\in [\gamma]} \bigvee\limits_{i\in G_h}  d_i^{(h)}(\bsigma;\theta) > \teps_h} + \Pr\sparen{\bigvee\limits_{h\in [\gamma]} \bigvee\limits_{i\in G_h} \max\limits_{\substack{i'\in G_{h'} \text{ s.t. }h > h'\land,\\
             \bsigma^{-1}(i) > \bsigma^{-1}(i')}} \theta_i - \theta_{i'}  > \teps_h}\\
            &\le \Pr\sparen{\bigvee\limits_{h\in [\gamma]} \bigvee\limits_{i\in G_h}  d_i^{(h)}(\bsigma;\theta) > \teps_h} + \Pr\sparen{\bigvee\limits_{h\in [\gamma]} \bigvee\limits_{i\in G_h} \bigvee\limits_{\substack{i'\notin G_h \text{ s.t. }h \neq h'\land,\\
            \bsigma^{-1}(i) > \bsigma^{-1}(i')}} \theta_i - \theta_{i'}  > \teps_h}\\
            &< \sum_{h \in [\gamma]} \sum_{i \in G_h} \frac{\delta}{2n\gamma} + \sum_{h \in [\gamma]} \sum_{i \in G_h} \sum_{h' > h} \sum_{i' \in G_{h'}} \frac{\delta}{2 n^2\gamma}\\
            &\le \sum_{h \in [\gamma]} \frac{\delta}{2\gamma} + \sum_{h \in [\gamma]} \frac{\delta}{2\gamma}\\
            &= \frac{\delta}{2} + \frac{\delta}{2}\\
            &= \delta.
        \end{align*}
       
\paragraph{Sample Complexity.}
\begin{itemize}
    \item \textbf{Group-wise rankings.} 
        Let $n_h = |G_h|$.
        From \Cref{thm:blind_ubound} we know that the sample complexity to find a group-wise ranking $\sigma_h$ for group $h \in [\gamma]$ is $\cO\paren{\frac{n_h}{\teps_h^2}\log\frac{n_h}{\delta}}$.
        Therefore, the sum of the sample complexities to find all the group-wise rankings is $\cO\paren{\sum_{h = 1}^\gamma \frac{n_h}{\teps_h^2}\log\frac{n_h}{\delta}}$.
    \item \textbf{Merging group-wise rankings.}
        The algorithm merges two lists at a time. Hence, it takes $O(\log \gamma)$ iterations of the \textbf{while} loop before the algorithm terminates.
        In each iteration of the \textbf{while} loop, we need to find a merged list, of length $n_h + n_{h'}$, the sum of the lengths of the two sorted lists that are being merged.
        To fill each position in the merged list, we make $O\paren{\frac{1}{\min\{\teps_h, \teps_{h'}\}^2}\log \frac{2n^2\gamma}{\delta}} = O\paren{\frac{1}{\min\{\teps_h, \teps_{h'}\}^2}\log \frac{2n\gamma}{\delta}}$ oracle calls, since $\gamma \le n$.
        Note that since the values of $\teps_h$ are different for different groups, merging different lists takes a different number of Oracle calls, which makes it hard to analyze the sample complexity. Hence, from now on, we will analyze the sample complexity using the smallest group-wise error. Let $g := \arg\min_{h \in [\gamma]} \teps_h$.
        Now, the sample complexity of merging two lists will be $O\paren{(n_h+n_{h'})\cdot \frac{1}{\teps_{g}^2}\log \frac{2n\gamma}{\delta}} = O\paren{(n_h+n_{h'})\cdot \frac{1}{\teps_{g}^2}\log \frac{n}{\delta}}$.
        Therefore, the total sample complexity for one iteration of the while loop can be upper bounded by $O\paren{\paren{\sum_{h \in [\gamma]}n_h}\cdot \frac{1}{\teps_{g}^2}\log \frac{n}{\delta}}$ which is $ O\paren{\frac{n}{\teps_{g}^2}\log \frac{n}{\delta}}$.
        Note that this is the sample complexity for every iteration of the \textbf{while} loop.
        Therefore, the total sample complexity of the algorithm can be upper bounded by $O\paren{\frac{n\log\gamma}{\teps_{g}^2}\log \frac{n}{\delta}}$.

\end{itemize}

Therefore, the total sample complexity is,
\begin{align*}
    \cO\paren{\paren{\sum_{h = 1}^\gamma \frac{n_h}{\teps_h^2}\log\frac{n}{\delta}}+\paren{\frac{n\log\gamma}{\teps_{g}^2}\log \frac{n}{\delta}}} 
\end{align*}
Note that neither of these terms strictly dominates the other for all problem instances.
Substituting $\teps_h = \frac{\epsilon_h}{2}\cdot \paren{\frac{2}{n_h}}^{1/p}$ we get that the sample complexity is,
\begin{align*}
    \cO\paren{\paren{\sum_{h = 1}^\gamma \frac{n_h^{1+\frac{2}{p}}}{\epsilon_h^2}\log\frac{n}{\delta}}+\paren{\frac{n\cdot n_g^{2/p}\cdot\log\gamma}{\epsilon_{g}^2}\log \frac{n}{\delta}}}.
\end{align*}

\end{proof}

\subsubsection{Lower Bound}
\label{sec:aware_lbound}

\awarelower
\begin{proof}
We begin by constructing the class of hard of instance where the algorithm needs to get the ordering of the items within the groups entirely correct according to their true scores, i.e., even one swap within the group will result in the overall error of $> \epsilon$.

\paragraph{Class of instances for group-aware algorithms.}
For a fixed group $h \in [\gamma]$, let $\epsilon_h := \epsilon\cdot\left(\frac{2n_h}{\phi_h\gamma}\right)^{1/q}$and $\teps_h = \epsilon_h\cdot\paren{\frac{4}{n_h}}^{1/p}$,  where $n_h = |G_h|$.
For any $h \in [\gamma]$, let $n_h' := 3n_h/4$.
Let $I_h := \{i\in G_h\}$ be an ordered index set of the items in group $h$ such that $I_h(j)$ gives the $j$th item in $I_h$.
We use $I_h[t]$ to represent the first $t$ items in $I_h$.
Let $T_h := \{I_h(n_h'+1),I_h(n_h'+2),\ldots, I_h(n_h)\}$.
For any ranking $\bsigma$, we use $\bsigma^{(h)}$ to represent the sub-ranking corresponding to group $h$, i.e., it gives us a ranking of the items only from group $h$ in the original ranking $\bsigma$, where $\forall i, i' \in G_h, \bsigma(i) < \bsigma(i')\iff \bsigma^{(h)}(i) < \bsigma^{(h)}(i')$.
Note that throughout the paper, we are only interested in the relative ordering of the items, hence, the exact ranks of the items in $\bsigma^{(h)}$ are not of concern.
That is, the last $n_h/4$ items from group $h$ are put in the set $T_h$.
Now consider the class of instances $\bnu_{[m_1, m_2, \ldots, m_\gamma]}$ for any $m_h \in [n_h']$ which contains instances represented as $\nu_S$ where $S := \bigcup_{h \in [\gamma]}S_h$ and $S_h \subseteq I_h(n_h')$ such that $|S_h| = m_h$ and the scores of the items for the instance $\nu_S$ are
\begin{multline*}
    \forall h \in [\gamma], \forall i \in S_h,~~\theta_i = \theta^{(h)}\paren{\frac{1}{2} + \teps_h}^2,~~\forall i \in T_h,~~\theta_i = \theta^{(h)}\paren{\frac{1}{4}-\teps_h^2},\\\text{and}\qquad\forall i \not\in S_h\cup T_h,~~\theta_i = \theta^{(h)}\paren{\frac{1}{2} - \teps_h}^2,
\end{multline*}
for some numbers $\theta^{(h)}$ for each group $h \in [\gamma]$ to be defined later.

\begin{remark}
\label{rem:group_unique_inst}
Note that $S$ uniquely represents an instance $\nu_S\in\bnu_{[m_1, m_2, \ldots, m_\gamma]}$ for any fixed $m_h\in [n_h'], \forall h \in [\gamma]$.
\end{remark}

\begin{lemma}
\label{lem:group_inst_not}
For any $\theta^{(h)} > \frac{1}{1-2\teps_h}$ for each group $h \in [\gamma]$ and for any problem instance $\nu_{S} \in \bnu_{[m_1, m_2, \ldots, m_h]}$ such that for each group $h \in [\gamma]$, $m_h = n_h/4$, any \epqbr, say $\bsigma_{S}$, has to satisfy the following: for at least half the number of groups,
the number of items from $S_h \cup T_h$ in ranks $\frac{n_h}{4}+2$ to $n_h$ in the sub-ranking of the items only from group $h$ in $\bsigma$ should be strictly less than $\frac{n_h}{4}$. 
\end{lemma}
\begin{proof}
    Note that when $\theta^{(h)} > \frac{1}{1-2\teps_h}$, for any $i \in S_h, i' \in T_h$ and $i'' \in G_h \sm (S_h \cup T_h)$, 
    \[
    \theta_i - \theta_{i'} > \teps,~~\theta_{i'} - \theta_{i''} > \teps,~~\text{and}~~\theta_{i} - \theta_{i''} > \teps.
    \]
    Let us assume that there exists an \epqbr~for $\nu_S$, say $\bsigma_S$, such that the ranks $\frac{n_h}{4}+2$ to $n_h$ have $\ge \frac{n_h}{4}$ items from $S_h\cup T_h$.
    Let $\Gamma \subseteq [\gamma]$ be the set of groups on which this happens.
    Then for each such group in $\Gamma$, $\le \frac{n_h}{4}$ items from $S_h \cup T_h$ are in the ranks $1$ to $\frac{n_h}{4}+1$ in the sub-ranking $\bsigma_S^{(h)}$, since $|S_h\cup T_h| = \frac{n_h}{2}$.
    So there will be at least one item from $I_h[n_h']\sm S_h$ in ranks $1$ to $\frac{n_h}{4}+1$ in $\bsigma^{(h)}_S$, which implies that at least $\frac{n_h}{4}$ items in $S_h \cup T_h$ incur an error $> \teps_h$.
    Therefore, the overall error for that group will be,
    \begin{align*}
        \err_h(\bsigma^{(h)}_S;\theta_S) &= \paren{\sum_{i\in[n]}d_i(\bsigma^{(h)}_S;\theta_S)^p}^{1/p} > \paren{\frac{n}{4}\cdot\teps^p}^{1/p}\\
        &= \paren{\frac{n}{4}\cdot \paren{\epsilon\cdot \paren{\frac{4}{n}}^{1/p}}^p}^{1/p}
        = \paren{\epsilon^p}^{1/p}
        = \epsilon,
    \end{align*}
    Then, 
    \begin{align*}
        \efair(\bsigma_S;\theta_S)
        &= \paren{\sum_{h \in [\gamma]}\frac{\phi_{h}}{n_{h}}\cdot \err_h(\bsigma^{(h)}_S; \theta_S)^q}^{1/q}
        \ge \paren{\sum_{h \in \Gamma}\frac{\phi_{h}}{n_{h}}\cdot \err_h(\bsigma^{(h)}_S; \theta_S)^q}^{1/q}\\
        &= \paren{\sum_{h \in \Gamma}\frac{\phi_{h}}{n_{h}}\cdot \epsilon_h^q}^{1/q}
        = \paren{\sum_{h \in \Gamma}\frac{\phi_{h}}{n_{h}}\cdot \paren{\frac{2n_h}{\phi_h\gamma}\cdot\epsilon^q}}^{1/q}\\
        &> \paren{\frac{\gamma}{2}\cdot \frac{2\epsilon^q}{\gamma}}^{1/q}\\
        &= \epsilon.
    \end{align*}
\end{proof}

\paragraph{The alternative instances.}
We now fix any set $S^* \subset \bigcup_{h \in [\gamma]} I_h[n_h']$ such that for half the number of groups $\Gamma \subset [\gamma]$, $|S^*\cap G_h|= \frac{n_h}{2} $ and for the other half $|S^*\cap G_h|= \frac{n_h}{4} $. 
Lower bound on the sample complexity is now obtained by applying \Cref{lem:gar16} on a pair of instances $(\nu_{S^*}, \nu_{\tS^*})$, for all possible choices of $\tS^* = \sS \cup S'$, where for exactly one group $h \in [\gamma] \sm \Gamma$, $S' \subset \paren{I_h[n_h'] \sm \sS}$ and $|S'\cap G_h| = \frac{n_h}{4}$.
Note that there will be $\sum_{h \in [\gamma]\sm \Gamma}{n_h/2 \choose n_h/4}$ choices of $\tS^*$.
\paragraph{Describing the event.}
For any ranking $\bsigma \in \Sigma_n$, we denote by $\bsigma(r:r')$ the set of items in the ranking, $\bsigma$, from rank $r$ to rank $r'$, for any $1 \le r \le r' \le n$.
Consider the event $\cE$ for an instance $S$ that the algorithm $\cA$ outputs a ranking such that for at least half the number of groups the following holds: for the sub-ranking corresponding to the group $h \in [\gamma]$, the number of items from $S_h \cup T_h$ in the ranks $n_h/4+2$ to $n_h$ is $< n_h/4$. That is,
\[
\cE(S) := \left|\left\{h \in [\gamma]: \left|(S_h \cup T_h) \cap  \bsigma^{(h)}_{\cA}\left(n_h/4+2:n_h\right)\right| < n_h/4\right\}\right| \ge \frac{\gamma}{2}.
\]
This is a high probability for $S^*$ because otherwise, the error will be more than $\epsilon$, from \Cref{lem:inst_not}.
On the contrary, for the alternative instances with $\tS^*$, this is a low probability event because for at least half the number of groups all the items from $T_h$ have to appear in the ranks $n_h/4+2$ to $n_h$ or otherwise, $\ge n_h/4$ items from $S_h$ will be in ranks $n_h/4+2$ to $n_h$ and they all incur an error $> \teps_h$ due to an item from $T_h$ in the ranks $1$ to $n_h/4+1$; therefore the total error will be more than $\epsilon$.

It is easy to note that as $\cA$ is an \pac\,, obviously 
\begin{equation}
Pr_{S^*}( \cE(\sS) ) > Pr_{S^*}(\bsigma_{\cA} \text{ is an \ebr}) > 1-\delta,\label{eq:group_high_prob}
\end{equation}
and
\begin{equation}
Pr_{\tS^*}( \cE(\sS) ) < \delta,\label{eq:group_low_prob}
\end{equation}
for any alternative instance $\tS^*$.

The way the class of instances is constructed, the algorithms need to satisfy $\cE_h(S_h)$ for groups $h \in [\gamma] \sm \Gamma$ because the algorithm making pairwise comparisons in group $h \in [\gamma]\sm \Gamma$ differentiates between the true instance and a distinct set of alternative instances, namely those where the scores of the items from group $h$ are altered, but it can not differentiate between any other altered instance and the true instance.
Therefore,
\begin{align*}
    Pr_{\tS^*}( \cE(S^*) ) = Pr_{\tS^*}\paren{\bigcup_{h \in [\gamma]\sm \Gamma}\cE_h(S_h^*)} = \sum_{h \in [\gamma]\sm \Gamma}Pr_{\tS_h^*}(\cE_h(S_h^*)).
\end{align*}
Therefore, if there exists even one group $h \in [\gamma]\sm \Gamma$ such that $Pr_{\tS_h^*}( \cE_h(S_h^*) ) > \delta$, it implies that $Pr_{\tS^*}( \cE(S^*) ) > \delta$, which contradicts that $\cA$ is an \pac.
Therefore, for every group $h \in [\gamma]\sm \Gamma$, the algorithm $\cA$ needs to satisfy that $Pr_{\tS_h^*}( \cE_h(S_h^*) ) < \delta$. Similar argument gives us that $Pr_{S_h^*}( \cE_h(S_h^*) ) > 1 - \delta$, for every group $h \in [\gamma]\sm \Gamma$.

We can further apply \Cref{lem:lb_sym} to get 
\begin{equation}
    Pr_{\tS_h^*}( \cE_h(S^*_h) ) < \frac{\delta}{{n_h/2 \choose n_h/4}}, \forall h \in [\gamma] \sm \Gamma.
\end{equation}
From \Cref{lem:kl_div} we then have that, $KL(\nu^B_{S^*_h}), \nu^B_{\tS_h^*}) \le 64\teps_h^2$ for every $h \in [\gamma] \sm \Gamma$ and for every $B \in \cB_h$.
Therefore, applying \Cref{lem:gar16} for every group $h \in [\gamma] \sm \Gamma$ we get,

\begin{align}
\label{eq:group_lb2}
\nonumber \sum_{\{B \in \cB_h\}} \E_{\nu^B_{S_h^*}}[N_B(\tau_A)]KL(\nu^B_{S_h^*},\nu^B_{\tS_h^*}) \le \sum_{\{B \in \cB_h\}} \E_{\nu^B_{S_h^*}}[N_B(\tau_A)]{\frac{n_h}{2}-1\choose \frac{n_h}{4}-1}64\teps^2.
\end{align}

Finally noting that $\tau_A^{(h)} \ge \sum_{B \in \cB_h}[N_B(\tau_A)]$, we get 

\begin{align*}
{\frac{n_h}{2}-1\choose \frac{n_h}{4}-1}64\teps^2\E_{\nu^B_{S_h^*}}[\tau_A^{(h)}] =  \sum_{S \in \cB_h}\E_{\nu^B_{S_h^*}}[N_B(\tau_A^{(h)})]({\frac{n_h}{2}-1\choose \frac{n_h}{4}-1}64\teps^2) \ge {\frac{n_h}{2}\choose \frac{n_h}{4}}\ln \frac{{\frac{n_h}{2}\choose \frac{n_h}{4}}}{4\delta}.
\end{align*}
Therefore,
\begin{align*}
    \E[\tau_A^{(h)}] &\ge \frac{1}{\teps_h^2}\frac{n_h/2}{n_h/4}\ln \frac{{\frac{n_h}{2}\choose \frac{n_h}{4}}}{4\delta} = \frac{2}{\teps^2}\ln \frac{{\frac{n_h}{2}\choose \frac{n_h}{4}}}{\delta} \\
    &\ge \frac{2}{\teps_h^2} \ln \frac{2^{n_h/2}\sqrt{6}}{\sqrt{\pi\paren{3n_h/2+2}}\delta} = \frac{n_h}{\teps_h^2} + \frac{1}{\teps_h^2}\ln\frac{6}{\pi\paren{3n_h/2+2}\delta} \ge \frac{n_h}{\teps_h^2}.
\end{align*}

Since $\tau_A \ge \sum_{h \in [\gamma]\sm \Gamma}\tau_A^{(h)}$, we get,
\begin{align*}
    \E[\tau_A] \ge \sum_{h \in [\gamma] \sm \Gamma} \E[\tau_A^{(h)}] \ge \sum_{h \in [\gamma] \sm \Gamma} \frac{n_h}{\teps_h^2} \ge \frac{1}{2}\sum_{h \in [\gamma]}\frac{n_h}{\teps_h^2}.
\end{align*}
Therefore, the sample complexity to figure out the group-wise rankings is $\Omega\paren{\sum_{h \in [\gamma]}\frac{n_h}{\teps_h^2}}$.

\end{proof}

\newpage
\section{Additional Experimental Results}
\label{app:experiments}
Following is a summary of the results presented here:
\begin{enumerate}
    \item \Cref{tab:proportions} shows the proportions of the groups in the top $n$ items. This helps visualize that the group-wise errors computed by our \aware~algorithm are proportional to the sizes of the groups (see \Cref{fig:group_errors}).
    \item \Cref{fig:syn_geo,fig:syn_arith,fig:syn_steps,fig:syn_har} show the results on the synthetic datasets whose true scores are as shown in \Cref{fig:syn_true_scores}.
    \item \Cref{fig:syn_group_errors} shows the group-wise errors for the synthetic datasets. Again, \aware\\
    achieves proportional errors across groups.
    \item \Cref{fig:syn_unweighted,fig:syn_group_errors_unweighted} show the results when we choose $\phi_h = n_h$ instead of $\phi_h = 1$.
\end{enumerate}

\begin{table}[h]
    \centering
    \begin{tabular}{|l|c|c|c|c|}
    \hline
        \textbf{Dataset} & $n$ & $G_0$ & $G_1$ & $G_2$\\
        \hline
        COMPAS-race & 25 & 88\% &  12\% & - \\
        COMPAS-race & 50 & 88\% &  12\% & - \\
        COMPAS-race & 100 & 87\% &  13\% & - \\
        \hline
        COMPAS-gender & 25 & 92\% &  8\% & - \\
        COMPAS-gender & 50 & 92\% &  8\% & - \\
        COMPAS-gender & 100 & 84\% &  16\% & - \\
        \hline
        German-age & 25 & 76\% &  16\% & 8\% \\ 
        German-age & 50 & 76\% &  16\% & 8\% \\ 
        German-age & 100 & 76\% &  15\% & 9\% \\ 
        \hline
    \end{tabular}
    \caption{Proportion of the groups in the top $n$ items in the real-world datasets.}
    \label{tab:proportions}
\end{table}

\begin{figure}[h]
    \centering
    \includegraphics[scale=0.13]{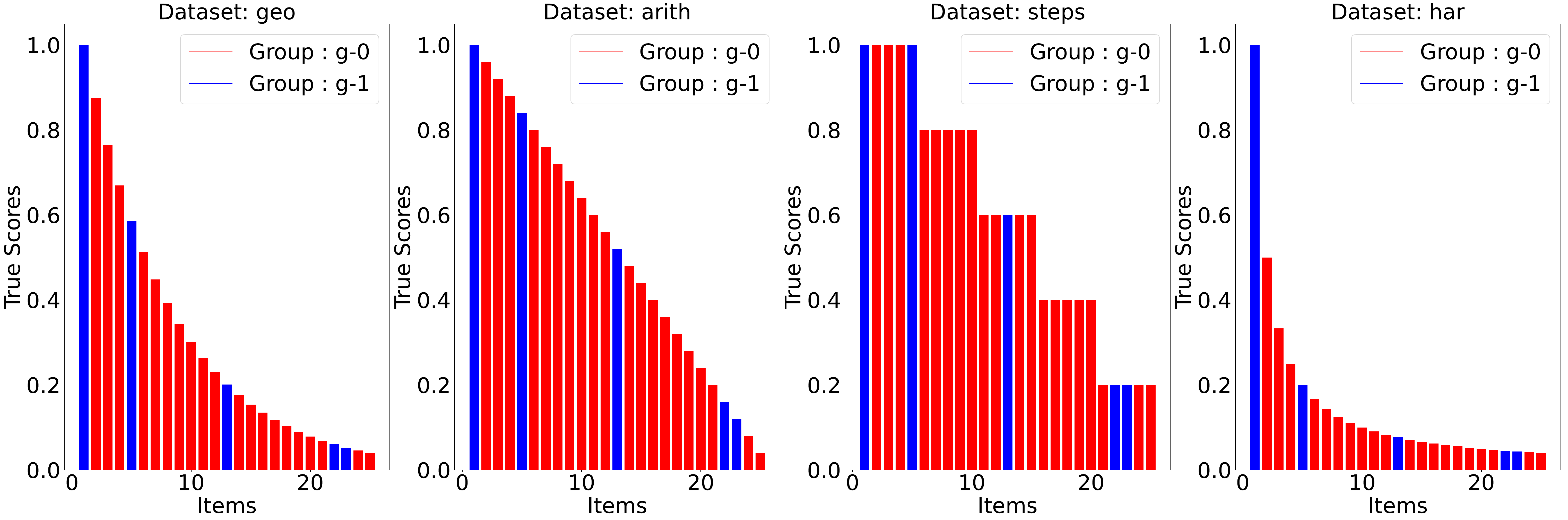}
    \caption{True scores of the synthetic datasets where in \textbf{geo} the scores decrease in a geometric progression, in \textbf{arith} the scores decrease in an arithmetic progression, in \textbf{steps} the scores decrease in an arithmetic progression but only for every 5 items, and for \textbf{har} the scores decrease in a harmonic progression. The colors of the bars represent the groups the items belong to.}
    \label{fig:syn_true_scores}
\end{figure}

\begin{figure}
    \centering
    \includegraphics[scale=0.25]{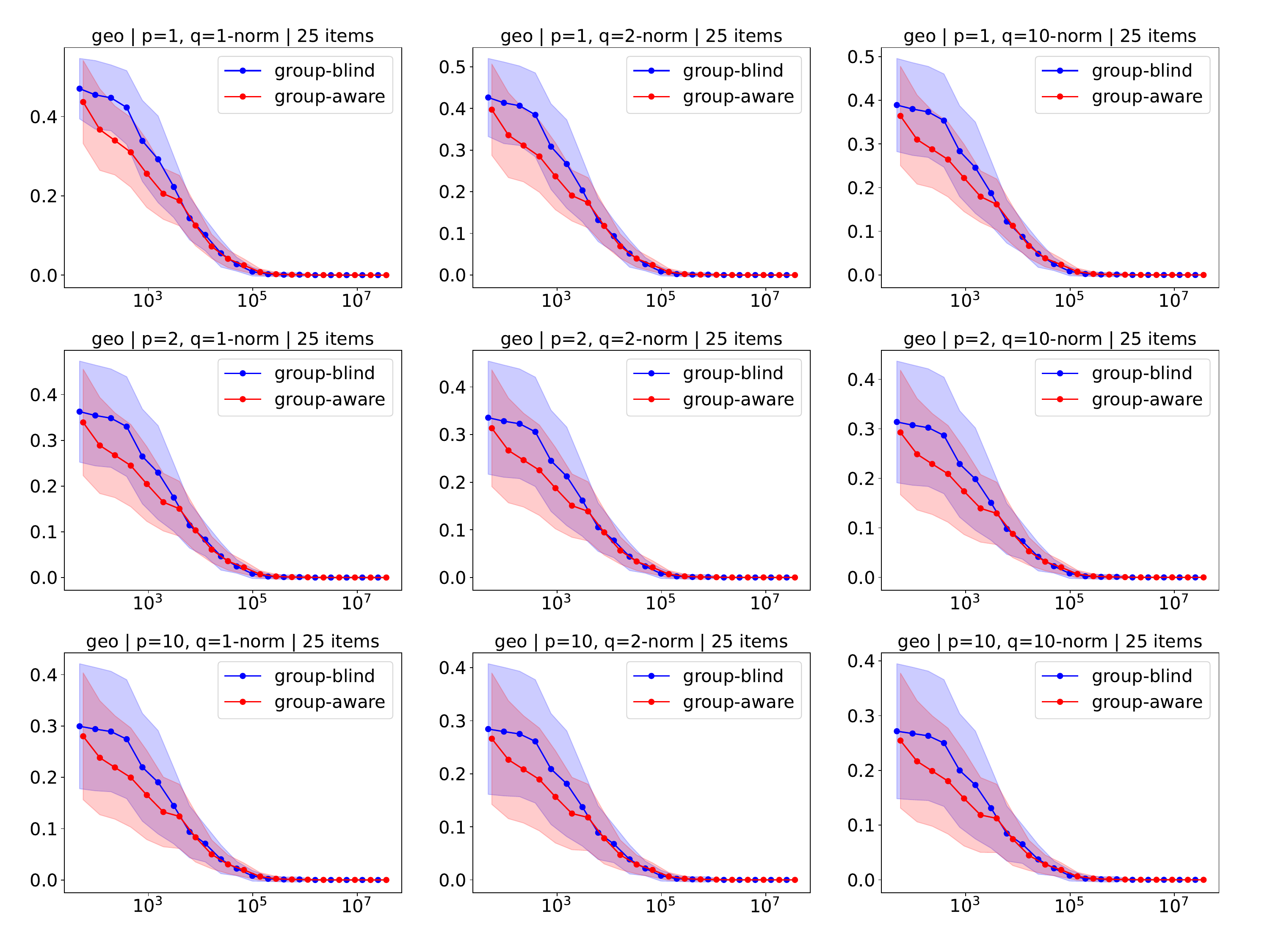}
    \caption{Group-wise errors (for $n=25$) for \textbf{geo} dataset for $g-0$ (majority group) and $g-1$ (minority group).}
    \label{fig:syn_geo}
\end{figure}

\begin{figure}
    \centering
    \includegraphics[scale=0.25]{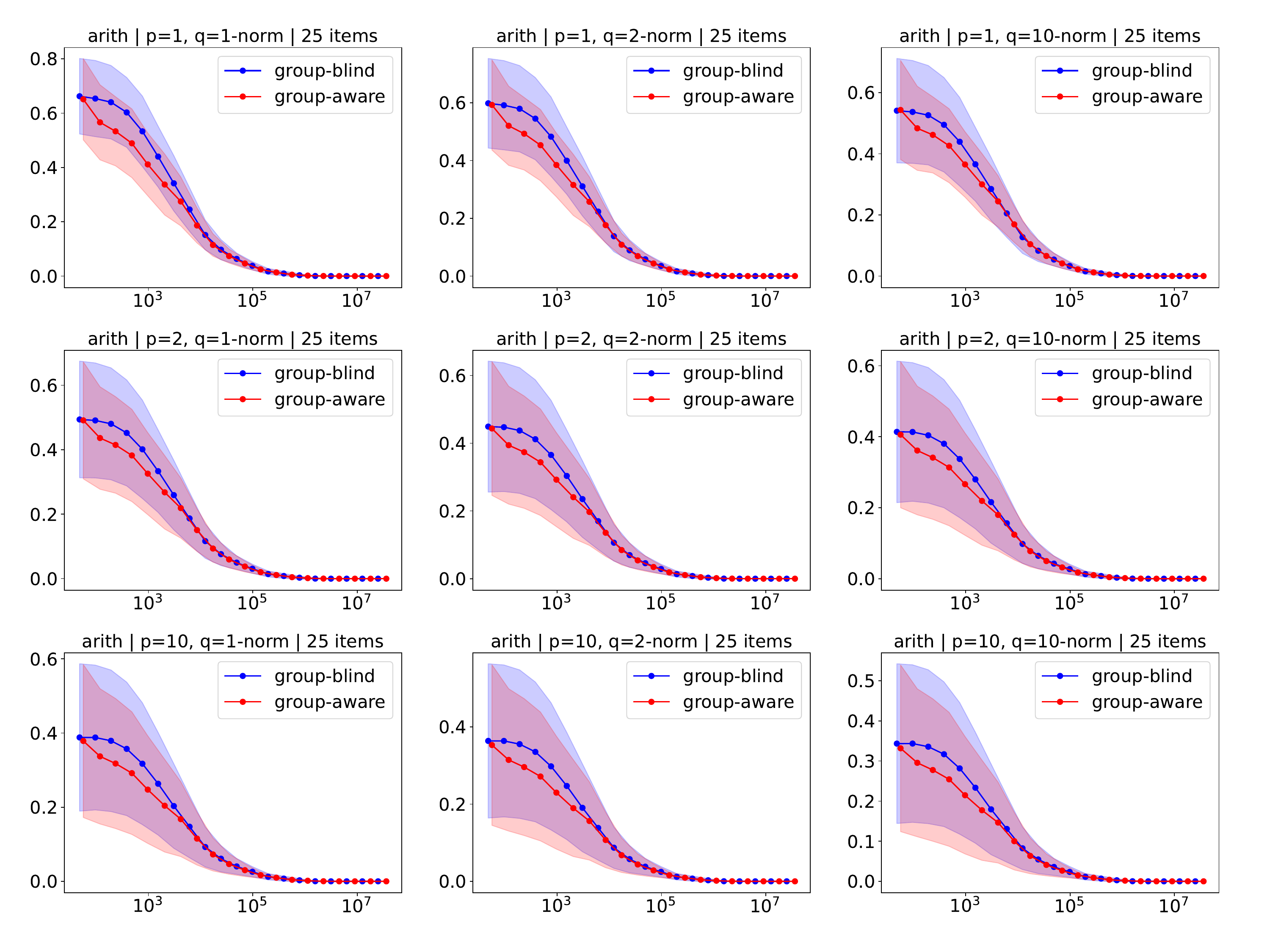}
    \caption{Group-wise errors (for $n=25$) for \textbf{arith} dataset for $g-0$ (majority group) and $g-1$ (minority group).}
    \label{fig:syn_arith}
\end{figure}

\begin{figure}
    \centering
    \includegraphics[scale=0.25]{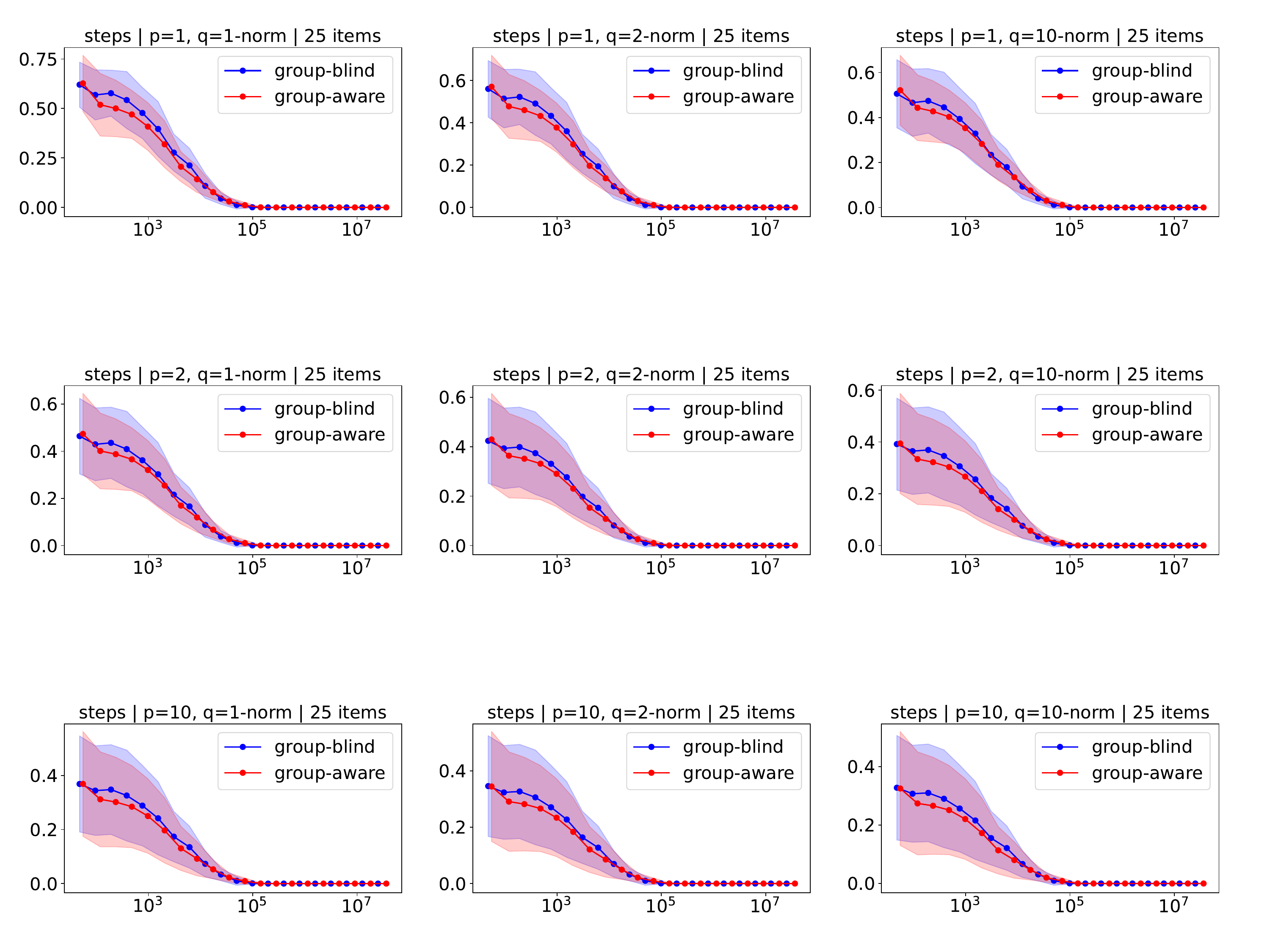}
    \caption{Group-wise errors (for $n=25$) for \textbf{steps} dataset for $g-0$ (majority group) and $g-1$ (minority group).}
    \label{fig:syn_steps}
\end{figure}

\begin{figure}
    \centering
    \includegraphics[scale=0.25]{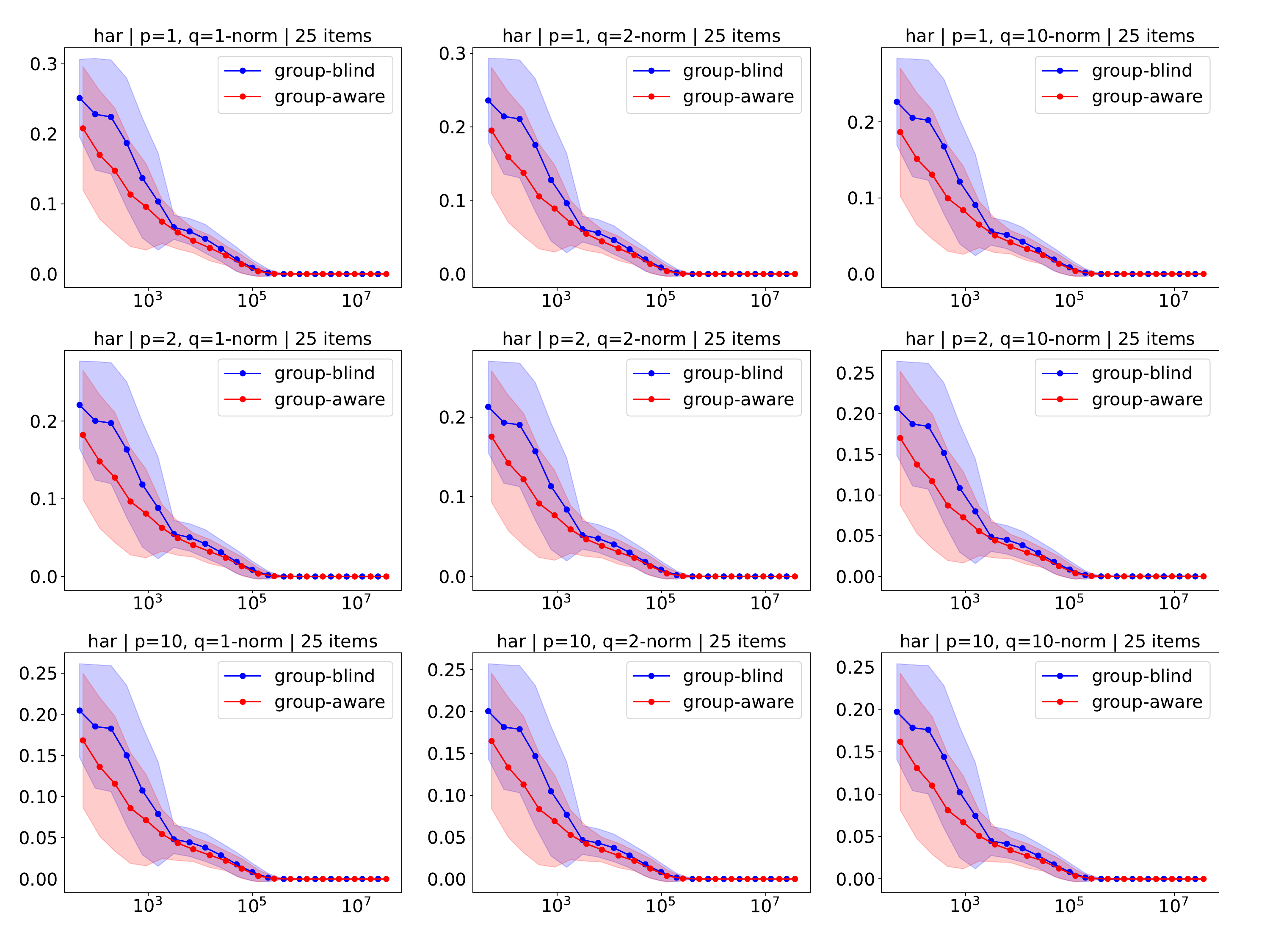}
    \caption{Group-wise errors (for $n=25$) for \textbf{har} dataset for $g-0$ (majority group) and $g-1$ (minority group).}
    \label{fig:syn_har}
\end{figure}

\begin{figure}
    \centering
    \includegraphics[scale=0.3]{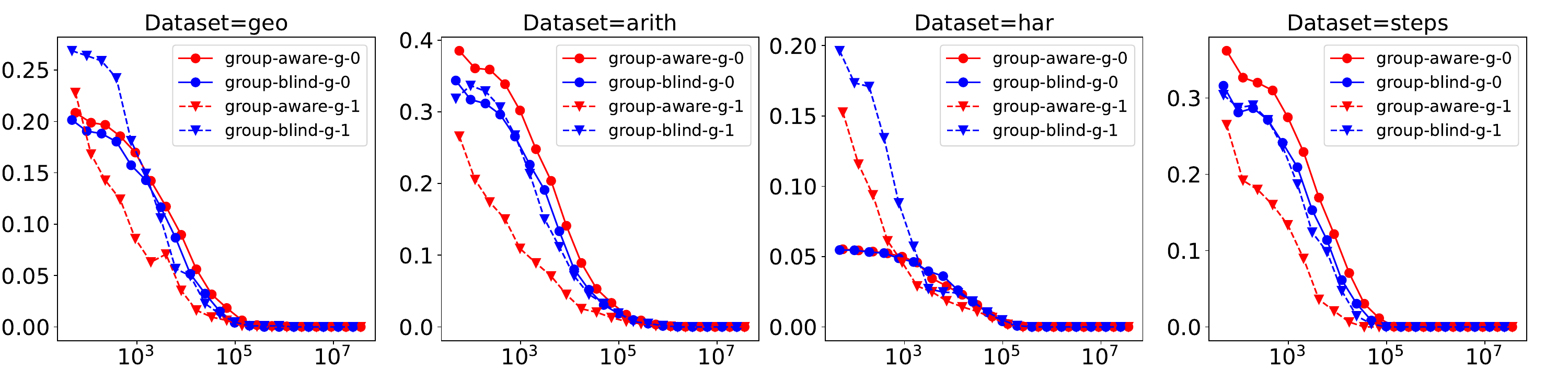}
    \caption{Group-wise errors (for $n=25, p=q=1$) for $g-0$ (majority group) and $g-1$ (minority group), for the synthetic datasets \textbf{(for $\phi_h = 1$)}.}
    \label{fig:syn_group_errors}
\end{figure}

\begin{figure}
    \centering
    \includegraphics[scale=0.2]{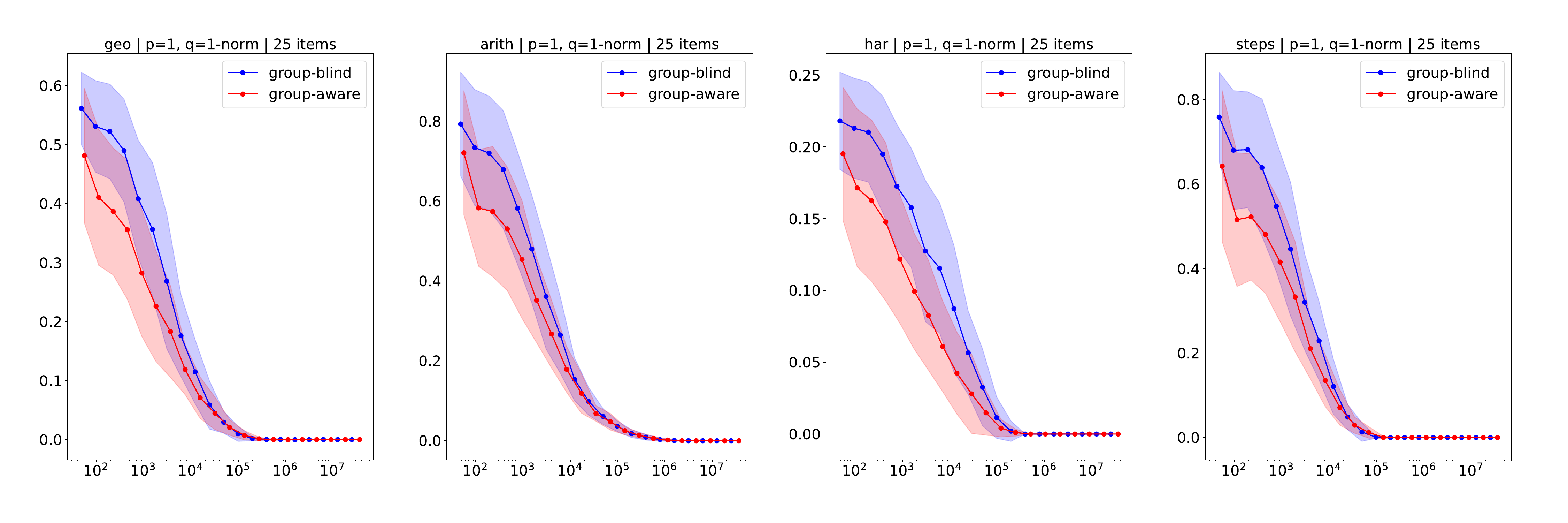}
    \caption{Group-wise errors (for $n=25$, $p=q=1$) for the synthetic datasets with $\phi_h = n_h$.}
    \label{fig:syn_unweighted}
\end{figure}

\begin{figure}
    \centering
    \includegraphics[scale=0.3]{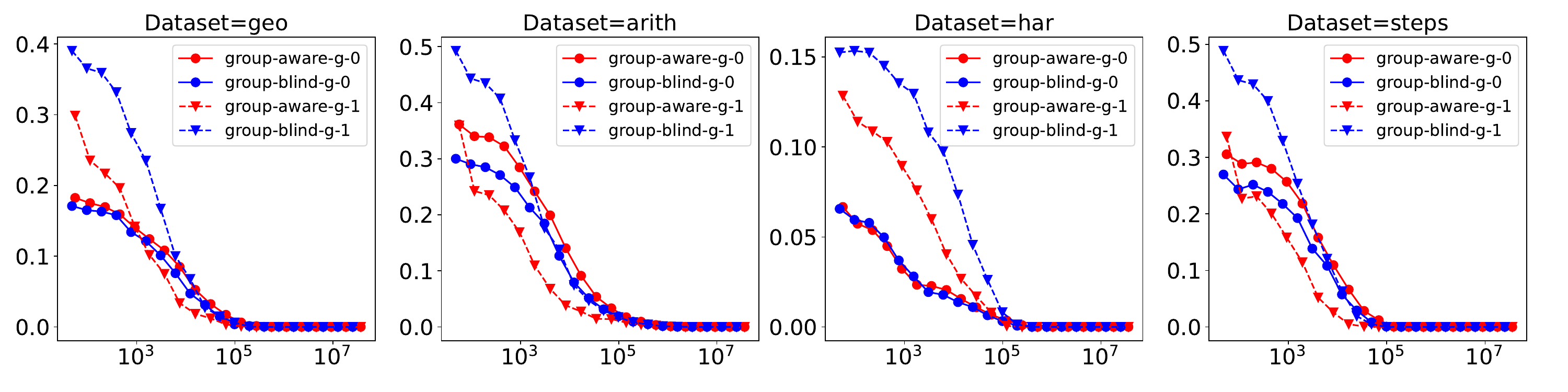}
    \caption{Group-wise errors (for $n=25, p=q=1$) for $g-0$ (majority group) and $g-1$ (minority group), for the synthetic datasets \textbf{(for $\phi_h = n_h$)}.}
    \label{fig:syn_group_errors_unweighted}
\end{figure}
\end{document}